\let\savedVert\|

\makeatletter
\def\cup@reference@code{%
  \RequirePackage[style=numeric,sorting=none,backend=biber]{biblatex}%
  \renewcommand*{\bibfont}{\footnotesize}%
}
\makeatother

\documentclass[
  journal=largetwo,
  manuscript=article,
  manuscriptlabel={Research Preprint},
  logo=false,
  simfonts=false,
  year=2026,
  volume=1,
]{cup-journal}

\let\|\savedVert

\usepackage[utf8]{inputenc}
\usepackage[T1]{fontenc}
\usepackage{amsmath,amssymb,amsthm}
\usepackage{mathtools}
  
\usepackage[varg]{newtxmath}
\DeclareMathSizes{10.95}{9.95}{7}{5}
\usepackage{graphicx} 
\usepackage{booktabs}
\usepackage{enumitem}
\setlist{nosep,leftmargin=*}
\usepackage{algorithm}
\usepackage{algpseudocode}
\usepackage{xcolor}
\usepackage{microtype}
\usepackage{tikz}
\usetikzlibrary{shapes,arrows.meta,positioning,calc,decorations.pathreplacing}
\usepackage[colorlinks=true,linkcolor=blue,citecolor=blue]{hyperref}
\usepackage{cleveref}
\usepackage{orcidlink}

\usepackage[section]{placeins}
\usepackage{stfloats}  

\usepackage{etoolbox}
\AtBeginEnvironment{table}{\footnotesize\setlength{\tabcolsep}{3pt}}

\definecolor{primary}{HTML}{006BA2}
\definecolor{accent}{HTML}{E3120B}
\definecolor{neutral}{HTML}{4A4A4A}

\definecolor{stable}{HTML}{2E7D32}
\definecolor{unstable}{HTML}{E3120B}
\definecolor{warning}{HTML}{F57C00}

\definecolor{boxfill}{HTML}{E8F4FD}
\definecolor{boxborder}{HTML}{006BA2}

\newtheorem{theorem}{Theorem}[section]
\newtheorem{lemma}[theorem]{Lemma}
\newtheorem{proposition}[theorem]{Proposition}

\theoremstyle{definition}
\newtheorem{definition}[theorem]{Definition}

\theoremstyle{remark}

\appto\normalsize{%
  \setlength{\abovedisplayskip}{6pt plus 2pt minus 2pt}%
  \setlength{\belowdisplayskip}{6pt plus 2pt minus 2pt}%
  \setlength{\abovedisplayshortskip}{2pt plus 1pt}%
  \setlength{\belowdisplayshortskip}{2pt plus 1pt}%
}
\normalsize
\AtBeginEnvironment{equation}{%
  \abovedisplayskip=6pt plus 2pt minus 2pt
  \belowdisplayskip=6pt plus 2pt minus 2pt
  \abovedisplayshortskip=2pt plus 1pt
  \belowdisplayshortskip=2pt plus 1pt
}

\makeatletter
\def\cup@journal@name{Convergence Radius}
\def\cup@manuscript{preprint}
\def\cup@year{February 2026}
\makeatother

\title{Ghosts of Softmax: Complex Singularities That Limit Safe Step Sizes in Cross-Entropy}
\author{Piyush Sao\,\orcidlink{0000-0002-9432-5855}}
\affiliation{Oak Ridge National Laboratory, Oak Ridge, TN, USA}
\email{saopk@ornl.gov}
\makeatletter\def\cup@contact@details{}\makeatother

\begin{document}

\begin{abstract}
Optimization analyses for cross-entropy training rely on local Taylor models of
the loss: linear or quadratic surrogates used to predict whether a proposed step
will decrease the objective. These surrogates are reliable only inside the
Taylor convergence radius of the true loss along the update direction. That
radius is set not by real-line curvature alone but by the nearest complex
singularity. For cross-entropy, the softmax partition function
$F=\sum_j \exp(z_j)$ has complex zeros---``ghosts of softmax''---that induce
logarithmic singularities in the loss and cap this radius. Beyond it, local
Taylor models need not track the true loss, so descent guarantees based on
those models become unreliable.

This yields a fundamentally different constraint from $L$-smoothness. To make
the geometry explicit and usable, we derive closed-form expressions under logit
linearization along the proposed update direction. In the binary case, the
exact radius is $\rho^*=\sqrt{\delta^2+\pi^2}/\Delta_a$. In the multiclass
case, we obtain the interpretable lower bound $\rho_a=\pi/\Delta_a$, where
$\Delta_a=\max_k a_k-\min_k a_k$ is the spread of directional logit derivatives
$a_k=\nabla z_k\cdot v$. This bound costs one Jacobian--vector product and
reveals what makes a step fragile: samples that are both near a decision flip
and highly sensitive to the proposed direction tighten the radius.

The normalized step size $r=\tau/\rho_a$ separates safe from dangerous updates.
Across six tested architectures and multiple step directions, no model fails for
$r<1$, yet collapse appears once $r\ge 1$. Temperature scaling confirms the
mechanism: normalizing by $\rho_a$ shrinks the onset-threshold spread from
standard deviation $0.992$ to $0.164$. A controller that enforces
$\tau\le\rho_a$ survives learning-rate spikes up to $10{,}000\times$ in our
tests, where gradient clipping still collapses. As a proof of concept, a
controller that sets $\eta=r\,\rho_a/\|v\|$ from local geometry alone reaches
$85.3\%$ accuracy on ResNet-18/CIFAR-10 without a hand-designed learning-rate
schedule (best fixed rate: $82.6\%$). Together, these results identify a
geometric constraint on cross-entropy optimization---one that operates through
Taylor convergence rather than Hessian curvature---and provide a tractable,
optimizer-agnostic bound that makes it visible.

\end{abstract}

\noindent\footnotesize\textbf{Reproducibility resources.}
Accompanying code, notebooks, and tutorials will be available at
\href{https://github.com/piyush314/ghosts-of-softmax}{github.com/piyush314/ghosts-of-softmax}.\par\normalsize

{%
\renewcommand{\thefootnote}{\fnsymbol{footnote}}%
\footnotetext{\noindent\rule{0.8\columnwidth}{0.4pt}\\[2pt]%
This manuscript has been authored by UT-Battelle,
  LLC under Contract No.\ DE-AC05-00OR22725 with the
  U.S.\ Department of Energy. The publisher, by accepting the
  article for publication, acknowledges that the United States
  Government retains a non-exclusive, paid-up, irrevocable,
  world-wide license to publish or reproduce the published form
  of this manuscript, or allow others to do so, for United States
  Government purposes. The Department of Energy will provide
  public access to these results of federally sponsored research
  in accordance with the DOE Public Access Plan
  (\url{http://energy.gov/downloads/doe-public-access-plan}).}%
}

\section{Introduction}
\label{sec:intro}

\begin{figure}[!t]
\centering
\includegraphics[width=\columnwidth]{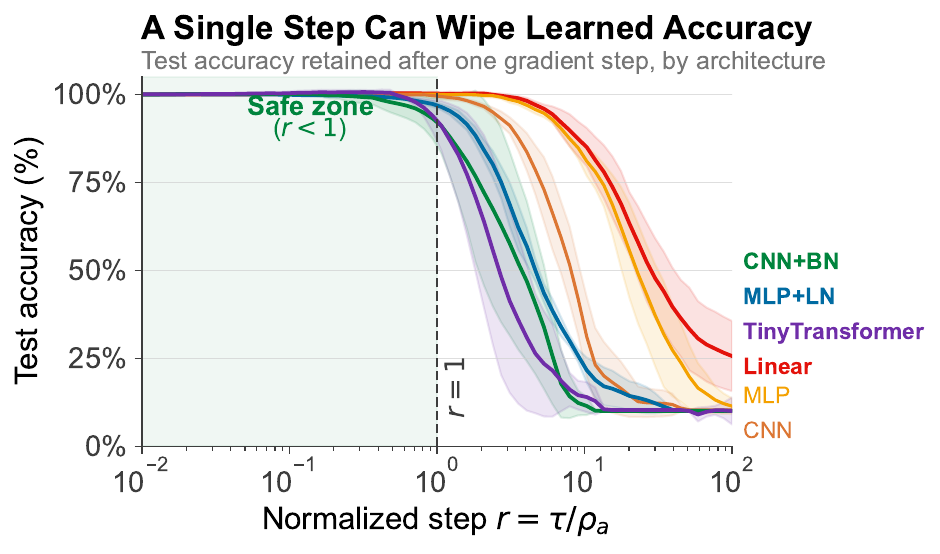}
\caption{
We train six small architectures to convergence, then take one step at varying
ratios $r = \tau/\rho_a$, where $\tau$ is the step distance and
$\rho_a$ is our stability bound (defined in the text).
Every network keeps its accuracy for $r < 1$;
once $r$ crosses 1, collapse appears across all six architectures.}
\label{fig:phase}
\end{figure}

Cross-entropy training works well but can fail suddenly. Loss spikes have been
reported in large-scale runs~\cite{chowdhery2022palm, touvron2023llama}, and a
learning rate that works early can diverge later. Practitioners mitigate these
spikes with heuristics:
restarting from checkpoints~\cite{chowdhery2022palm},
manual intervention~\cite{touvron2023llama},
gradient clipping~\cite{pascanu2013difficulty,zhang2020adaptive},
learning rate warmup and decay~\cite{loshchilov2017sgdr},
or temperature scaling~\cite{szegedy2016rethinking,pereyra2017regularizing}.

These heuristics help, but they treat symptoms, not causes. Clipping limits
gradient magnitude, not step safety. Learning rate scheduling reduces the rate
globally, slowing training everywhere. Standard $L$-smoothness theory misses
the key case: it assumes the gradient is Lipschitz with constant~$L$ and, for
$\eta < 2/L$, guarantees descent. Yet for cross-entropy this bound hides a
late-training failure mode. Figure~\ref{fig:phase} shows the result: at
roughly twice the later-defined bound $\rho_a$, accuracy collapses in one
update.

We identify a geometric mechanism that existing analyses miss. Many descent analyses justify an update by studying a local Taylor
polynomial of the loss. They treat this polynomial as a proxy for the true
objective and show that it decreases.
This reasoning is reliable only if the Taylor series converges at the update
point. By Cauchy--Hadamard (Section~\ref{sec:prelim}), convergence holds
inside the disk set by the nearest complex singularity. Inside the disk, the
truncated series tracks the function, so descent results transfer. Outside
the disk, the series can diverge from the function, so proving that the
series decreases no longer proves that the function decreases.

For cross-entropy loss, a key mechanism is the softmax partition function.
Along the step direction, the loss equals
$\log\!\bigl(\sum_k \exp z_k(\theta + \tau v)\bigr)$.
This sum stays positive for real~$\tau$ but has zeros in the complex plane.
Those zeros induce off-axis logarithmic singularities---``ghosts''---that
still cap the convergence radius.
As training sharpens predictions, the logit-derivative spread $\Delta_a$
often grows, moving ghosts toward the origin and shrinking the safe step
size. This tightening occurs even when $L$-smoothness permits larger
nominal steps.

For softmax
cross-entropy, linearizing the logits along the step direction yields the
closed-form lower bound $\rho_a = \pi/\Delta_a$, where
$\Delta_a = \max_k a_k - \min_k a_k$ is the logit-derivative spread
for $a_k = \nabla z_k \cdot v$.
Linearization is not needed for the mechanism; it is needed for a cheap and
interpretable bound. Without it, one can use numerical methods to locate the
nearest complex singularity. This first-order proxy is useful not only because it is computable (one JVP),
but because it shows which quantity matters: $\Delta_a$.
For binary cross-entropy the exact radius is
\begin{equation*}
\rho^* = \sqrt{\delta^2 + \pi^2}\;\big/\;\Delta_a,
\end{equation*}
where $\delta$ is the logit gap between classes; the worst case is
$\delta = 0$, giving $\rho_a$.
During training, $\Delta_a$ grows along the gradient direction,
so $\rho_a$ shrinks and a fixed learning rate eventually violates $r < 1$.

Define the normalized step size $r = \tau/\rho_a$, where
$\tau = \|p\|$ is the step distance.
When $\rho_a$ is recomputed along $v$ at each step,
$r < 1$ guarantees Taylor convergence for any optimizer direction
when logits are approximately linear over the step.
We confirm this direction-independence empirically: sweeping $r$ along the gradient
and 20 random directions yields the same transition at $r \approx 1$
(Section~\ref{sec:control}, Figure~\ref{fig:random}).
The bound is a reliability boundary, not a cliff:
training may survive $r > 1$ when cancellation is favorable,
but without a guarantee.

Our argument has three layers. First, a geometric fact: the real loss
$\ell(\tau) = \mathcal{L}(f(\theta + \tau v))$ is analytic in a complex
neighborhood of the real line, and its Taylor convergence radius is set by the
nearest complex singularity---not by real curvature. This holds for any
parameterization, independently of logit linearity.
Second, a tractable specialization: linearizing the logits along the step
direction isolates the softmax singularity structure, yielding the explicit
bound $\rho_a = \pi/\Delta_a$. Linearization makes the obstruction computable and interpretable.
Third, empirical evidence: this conservative lower bound is predictive enough
to diagnose, explain, and prevent instability across the settings we test.

\subsection*{Contributions}

\begin{enumerate}
\item \textbf{A geometric constraint on optimization
      (Sections~\ref{sec:prelim}--\ref{sec:radius}).}
      The real loss $\ell(\tau) = \mathcal{L}(f(\theta + \tau v))$ along any
      update direction~$v$ is analytic, and its Taylor convergence radius is
      set by the nearest complex singularity of~$\ell$.
      For cross-entropy, the softmax partition function
      $\sum_k \exp z_k(\tau)$ has complex zeros that induce logarithmic
      singularities in the loss. Beyond this radius, polynomial models of
      the loss can diverge from the function itself, so descent guarantees
      built on those models lose validity. This constraint operates through
      Taylor convergence geometry, not Hessian curvature, and tightens as
      predictions sharpen---even when the real loss surface appears
      increasingly flat.

\item \textbf{A tractable lower bound
      (Theorems~\ref{thm:binary}--\ref{thm:general}).}
      Linearizing the logits isolates the softmax singularity structure
      and yields $\rho_a = \pi/\Delta_a$, computable via one
      Jacobian--vector product. For binary cross-entropy the exact radius
      is $\rho^* = \sqrt{\delta^2 + \pi^2}/\Delta_a$.
      Linearization makes the softmax obstruction computable
      by revealing the controlling variable~$\Delta_a$.
      A separate proof via real-variable KL divergence bounds
      confirms the same $O(1/\Delta_a)$ scaling (\cref{app:kl}).

\item \textbf{Empirical validation
      (Sections~\ref{sec:predict}--\ref{sec:control}).}
      Across six architectures and multiple update directions,
      the normalized step $r = \tau/\rho_a$ separates safe from
      failing updates: no tested model fails for $r < 1$.
      Temperature scaling collapses the spread of failure-onset
      thresholds across architectures from
      $\sigma = 0.992$ to $0.164$.
      As a proof of concept, a controller that caps $\tau \le \rho_a$
      survives $10{,}000\times$ learning-rate spikes where gradient
      clipping collapses, and on ResNet-18/CIFAR-10 reaches $85.3\%$
      without a hand-designed schedule (best fixed rate: $82.6\%$).
\end{enumerate}

This paper establishes the one-step singularity constraint for softmax
and provides a tractable bound. Multi-step dynamics, activation
singularities, and computational optimization are natural extensions
that build on this foundation (\cref{sec:activations,sec:discuss}).

Section~\ref{sec:prelim} reviews the complex-analytic foundations.
Section~\ref{sec:setup} formulates the problem.
Section~\ref{sec:radius} derives the convergence radius.
Sections~\ref{sec:predict}--\ref{sec:control} apply and validate the
framework. Section~\ref{sec:discuss} discusses implications and limitations.

\section{Preliminaries}
\label{sec:prelim}

\begin{figure}[t]
\centering
\includegraphics[width=\columnwidth]{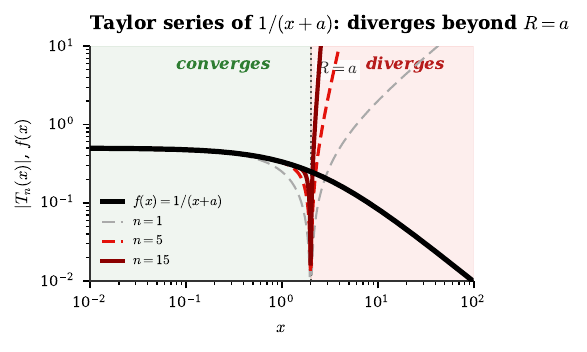}
\caption{Taylor approximations $T_n(x)$ of $f(x) = 1/(x+a)$ around $x_0 = 0$.
Inside the convergence radius $R = a$ (green), all orders approximate $f$ well.
Beyond $R$ (red), higher-order approximations diverge faster, not slower.}
\label{fig:taylor}
\end{figure}

\begin{figure}[t]
\centering
\includegraphics[width=\columnwidth]{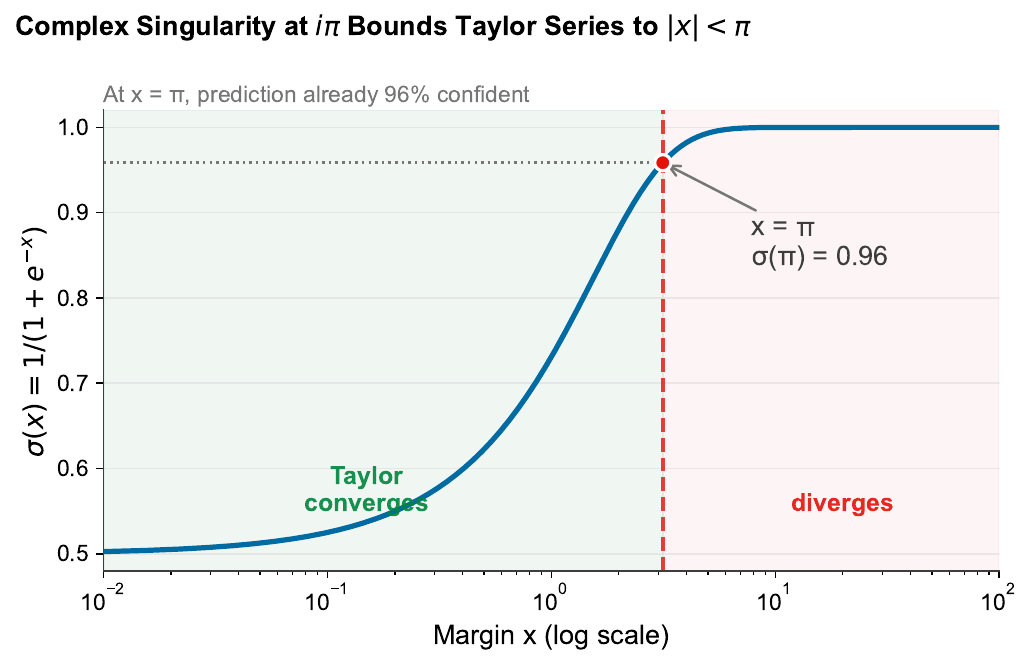}
\caption{Binary cross-entropy's log-partition $\log(1 + e^x)$ has convergence
radius $\rho = \pi$ set by the complex zero at $i\pi$ (Euler:
$1 + e^{i\pi} = 0$). Shown: derivative $\sigma(x)$ on log scale. Green:
Taylor converges ($|x| < \pi$). Red: diverges ($|x| > \pi$).}
\label{fig:sigmoid}
\end{figure}

This section develops a central insight used throughout the paper: many
optimization analyses and heuristics rely on local Taylor models, but those
models are reliable only within a convergence region that simple quadratic
approximations do not capture. We introduce concepts from complex analysis to
characterize this region and show how it constrains optimization.

\subsection{Analytic Functions and Taylor Series}

A function $f: \mathbb{R} \to \mathbb{R}$ is \emph{analytic} at $x_0$ if it
equals its Taylor series in some neighborhood:
\begin{equation}
f(x) = \sum_{n=0}^{\infty} \frac{f^{(n)}(x_0)}{n!}(x - x_0)^n, \quad |x - x_0| < R
\end{equation}
for some radius $R > 0$, and truncating at order $n$ then yields an
approximation with controlled error. This justifies finite-derivative models
such as linear ($n=1$) and quadratic ($n=2$, as in $L$-smoothness)
approximations.

\subsection{A Simple Example: \texorpdfstring{$f(x) = 1/(x+a)$}{f(x) = 1/(x+a)}}

To see how the convergence radius limits Taylor approximations, consider
$f(x)=1/(x+a)$ for $a>0$. On the positive reals, $f$ is smooth, bounded,
and monotone; its singularity lies at $x=-a$ on the negative half-line.
The Taylor series around $x_0=0$ is:
\begin{equation}
\label{eq:geom}
f(x) = \frac{1}{a} \sum_{n=0}^{\infty} \left(-\frac{x}{a}\right)^n
     = \frac{1}{a} - \frac{x}{a^2} + \frac{x^2}{a^3} - \cdots
\end{equation}
This converges only for $|x|<a$. Yet at $x=2a$---outside this
range---the function is perfectly well-defined: $f(2a)=1/(3a)$.
The Taylor series diverges; partial sums oscillate and grow in
magnitude. In fact, no finite truncation approximates $f(2a)$, and the
error can grow with truncation order (Figure~\ref{fig:taylor}).

\subsubsection{When does this divergence occur?}
Why does the singularity at $x=-a$ affect the series at $x=2a$? Along
the real line, the singularity seems far away and the function is smooth
in between.
The answer lies in the complex plane.

\begin{theorem}[Cauchy--Hadamard]
The Taylor series of $f$ around $x_0$ converges for $|x - x_0| < R$,
where $R$ is the distance from $x_0$ to the nearest point where $f$
fails to be analytic (a singularity)~\cite{conway1978functions}.
The series diverges for $|x - x_0| > R$.
\end{theorem}

For $f(x)=1/(x+a)$, the complex extension $f(z)=1/(z+a)$ has a pole at
$z=-a$. From the expansion point $z_0=0$, this pole lies at distance $a$, so
the convergence disk has radius $a$. The radius is determined by the complex
singularity, not by the function's behavior on the real line.

\subsubsection{Implications for neural networks}
This analyticity principle has practical consequences. The binary
log-partition $\log(1 + e^x)$ is smooth on $\mathbb{R}$ and has no real
singularities, yet its Taylor series diverges for $|x| > \pi$. The reason lies
in the complex plane: Euler's identity gives $1 + e^{i\pi} = 0$, so the
function vanishes at $x = i\pi$, creating a branch point of the logarithm.
This purely imaginary singularity limits the convergence radius to $\pi$
(Figure~\ref{fig:sigmoid}).

\paragraph{Key insight.}
A singularity need not lie between the expansion and evaluation points, nor
must it lie on the real line. If a singularity lies within distance $R$ of the
expansion point in the complex plane, the Taylor series diverges beyond $R$ in
every direction. This leads to the optimization question: when does an update
step exceed the convergence radius?

\subsection{Optimization and the Convergence Radius}

Any analysis that extrapolates local derivatives to predict finite-step
behavior faces the same limitation. We now make it precise for directional
steps.

\subsubsection{Taylor expansion of the update}
Let $f: \mathbb{R}^n \to \mathbb{R}$ be analytic, with gradient
$g = \nabla f(x)$ and Hessian $H = \nabla^2 f(x)$ at the current
iterate $x$. An optimizer chooses a step $p$ (e.g.\ $p = -\eta g$
for gradient descent, or a preconditioned direction for Adam).
The Taylor expansion around $x$ for any step $p$ is:
\begin{equation}
f(x + p) = f(x) + g^\top p + \tfrac{1}{2} p^\top H p + O(\|p\|^3).
\end{equation}
For gradient descent ($p = -\eta g$) this becomes:
\begin{equation}
\label{eq:gd-taylor}
f(x - \eta g) = f(x) - \eta \|g\|^2
    + \tfrac{1}{2} \eta^2 g^\top H g + O(\eta^3 \|g\|^3).
\end{equation}
The first-order term $-\eta \|g\|^2$ suggests descent. The second-order term
captures curvature along $g$, explaining why large $\eta$ can fail. But the
key constraint applies to any step $p$: the expansion is valid only when it
converges.

\subsubsection{The implicit assumption}
Equation~\eqref{eq:gd-taylor} gives a local expansion. To use this
approximation at a finite step, we must control higher-order terms. For
analytic $h(t)=f(x+t u)$ with $u=p/\|p\|$, let $\rho$ be the convergence
radius of the Taylor series around $t=0$. By Cauchy--Hadamard, series
convergence requires:
\begin{equation}
\label{eq:step-constraint}
\|p\| < \rho.
\end{equation}
For gradient descent ($p = -\eta g$), this yields
$\eta < \rho/\|g\|$ along direction $u=-g/\|g\|$. More generally:
\begin{equation}
\boxed{\|p\| < \rho}
\end{equation}
When the gradient is large (far from a stationary point), the permissible step
size shrinks. This is a separate mechanism from curvature-based bounds: the
local polynomial model itself becomes unreliable.

\subsubsection{Inside vs.\ outside the radius}
This unreliability is not gradual---the convergence radius $\rho$ creates a
sharp dichotomy:
\begin{itemize}
\item \textbf{Inside} ($\|p\| < \rho$): The Taylor series converges.
      Adding terms improves accuracy.
\item \textbf{Outside} ($\|p\| > \rho$): The Taylor series diverges.
      Adding terms need not improve the approximation and can make it
      \emph{worse} (Figure~\ref{fig:taylor}).
\end{itemize}
This is not a failure of finite precision or truncation error. It is a
fundamental limitation: once singularities limit convergence, derivatives at
$x$ do not fully encode behavior at the update point.

\subsubsection{Curvature versus convergence radius}
This fundamental limit differs from the familiar curvature-based bound.
Standard analysis uses the Lipschitz constant $L$ of the gradient to bound
step size, yielding $\eta < 2/L$. This curvature-based bound reflects how
rapidly the gradient changes. In contrast, the convergence radius $\rho$
imposes a separate, often stricter limit based on \emph{analyticity}: where
the Taylor representation breaks down entirely.

For functions like $f(x) = 1/(x+a)$, the Hessian $f''(x) = 2/(x+a)^3$ shrinks
as $x$ grows, suggesting that large steps are safe. Yet the convergence radius
remains $\rho = x + a$, forbidding steps that cross the singularity.
Curvature-based analysis misses this barrier; for cross-entropy, the mismatch
grows exponentially with confidence (Section~\ref{sec:predict}).

\subsubsection{Directional convergence radius}
So far we have treated the convergence radius as a scalar, but optimization
steps follow specific directions. Let $u$ be any unit direction
in parameter space. The \emph{directional convergence radius} along $u$ is:
\begin{equation}
\boxed{\rho_u = \sup\{\, t > 0 : \text{Taylor series converges at } x + tu
\,\}}
\end{equation}
This gives the constraint $\|p\| < \rho_u$ where $u = p/\|p\|$. Different
directions have different radii, determined by singularities in $f$'s complex
extension. The gradient direction maximizes first-order decrease, not
$\rho_u$, so these objectives may conflict.

\paragraph{What remains.} The preceding argument explains \emph{why} the
convergence radius matters: beyond it, no finite polynomial model guarantees
local fidelity. What remains is to \emph{compute} it. For cross-entropy, the
bottleneck is a log-partition zero at imaginary distance proportional to $\pi$
(Figure~\ref{fig:sigmoid}). The next section derives a computable lower bound
on this radius and shows how to use it.

\section{Problem Formulation}
\label{sec:setup}

We compute the radius of convergence for a given update direction. Along
that direction, this radius measures the distance to the nearest complex
singularity of the loss. Practically, it gives the maximum safe step
size implied by derivative information. We now formulate the problem and
identify the function whose zeros determine this radius.

\subsection{Model, Loss, and Movement}

To analyze the loss trajectory, we first define the model outputs and update
scheme.

\subsubsection{The model and logits}
A neural network $f_\theta$ maps input $x$ and parameters $\theta$ to raw,
unnormalized scores called \emph{logits}:
\[
z = f_\theta(x) \in \mathbb{R}^n
\]
where $n$ is the number of classes and $z_k$ is the score for class $k$.

\subsubsection{Cross-entropy loss}
For a target class $y$, the loss is the negative log-probability of that class:
\[
\mathcal{L} = -\log \mathrm{softmax}_y(z) = -z_y + \log \sum_{k=1}^n e^{z_k}
\]
The first term $-z_y$ is the correct-class logit. The second term
$\log \sum_k e^{z_k}$ is the \emph{log-partition function}---it normalizes
probabilities and introduces the nonlinearity whose behavior depends on the
step direction.
\subsubsection{The optimizer step}
To minimize loss, we update parameters by taking a step $\tau$ along a
unit direction $v$:
\[
\theta(\tau) = \theta_0 + \tau v
\]
where $\theta_0$ is the current parameters, $v$ is any unit vector in
parameter space (e.g.\ the negative normalized gradient, a normalized Adam
update, or a random
direction), and $\tau$ is the step distance.

\subsubsection{The 1D loss landscape}
We restrict analysis to the line defined by this step. Define the scalar
function $\ell(\tau)$ as the loss at distance $\tau$:
\[
\ell(\tau) = \mathcal{L}(\theta_0 + \tau v)
\]
Our goal is to determine the maximum $\tau$ for which a Taylor series
of $\ell(\tau)$ converges.

\subsection{The Logit Linearization Assumption}

Because the loss landscape is nonlinear, we approximate logits as linear in
the step size. This requires computing how each logit changes with the step.

\subsubsection{Jacobian-vector product}
The rate of change of logit $k$ along direction $v$ is
\[
a_k = \frac{d z_k}{d \tau}\bigg|_{\tau=0} = (\nabla_\theta z_k) \cdot v
\]
For all classes simultaneously, this becomes a Jacobian-vector product:
\[
a = J_z \cdot v \in \mathbb{R}^n
\]
where $J_z = \partial z / \partial \theta$ is the Jacobian matrix containing all
partial derivatives of logits with respect to parameters.

\subsubsection{Core assumption: linearized logits}
We assume logits change linearly for small steps:
\[
z_k(\tau) \approx z_k(0) + a_k \tau
\]
Substituting into the loss yields the approximated 1D loss:
\[
\ell(\tau) = -z_y(0) - a_y \tau + \log \sum_{k=1}^n e^{z_k(0) + a_k \tau}
\]
Define weights $w_k = e^{z_k(0)} > 0$ from the initial logits. The loss
simplifies to
\[
\ell(\tau) = -a_y \tau + \log F(\tau) + \mathrm{const}
\]
where
\begin{equation}
\label{eq:partition}
F(\tau) = \sum_{k=1}^n w_k e^{a_k \tau}
\end{equation}
is the \emph{partition function} along the step direction.

\subsubsection{The radius from partition zeros}
Only $\log F(\tau)$ in $\ell(\tau)$ can diverge, which happens when
$F(\tau) = 0$. The linear term $-a_y \tau$ has no singularities, so divergence depends
entirely on whether $F$ can vanish. For real
$\tau$, all terms are positive, so $F(\tau) > 0$. However, zeros can occur for
complex $\tau$. By Cauchy--Hadamard, the convergence radius is determined by
these complex zeros:
\[
\rho = \min\{|\tau| : F(\tau) = 0,\; \tau \in \mathbb{C}\}
\]
This is the central result of this section: under the linearization
assumption, the convergence radius equals the magnitude of the nearest complex
zero of the partition function. Any step $\tau < \rho$ lies within the
region where derivative information reliably predicts the loss, so $\rho$
is the maximum safe step size implied by the local Taylor model.

\paragraph{Sign invariance.}
This quantity depends on the directional derivatives $a_k$ but not on the sign
convention for the step: using $\theta_0 - \tau v$ flips all $a_k$ signs, but
$\Delta_a = \max_k a_k - \min_k a_k$ and all radius bounds are invariant.

\section{The Convergence Radius}
\label{sec:radius}

Section~\ref{sec:setup} showed that the convergence radius is determined
by zeros of the partition function $F(\tau) = \sum_k w_k e^{a_k \tau}$.
We now derive this radius step by step: first the exact value $\rho^*$,
then a closed form for binary classification, then a per-sample lower
bound $\pi/\Delta_a$, and finally batch and dataset bounds.

\subsection{Definition of \texorpdfstring{$\rho^*$}{rho*}}

\begin{definition}[Exact Convergence Radius]
\label{def:exact-radius}
For a loss function $f$ at parameters $\theta$ with unit update direction
$v$, the \emph{exact convergence radius} is the convergence radius
of $\ell(t) = f(\theta + tv)$ at $t = 0$. By the Cauchy--Hadamard theorem, this
equals the distance from the origin to the nearest singularity of $\ell$'s
analytic continuation to $\mathbb{C}$.
\end{definition}

This definition applies to any loss; for cross-entropy specifically, the
singularities come from zeros of the partition function $F$ (equivalently,
singularities of $\log F$).

\begin{proposition}[Radius from Partition Zeros]
\label{prop:radius}
For cross-entropy loss with partition function $F(t) = \sum w_k e^{a_k t}$
(where $w_k > 0$), the exact radius is
\begin{equation}
\rho^* = \min\{|t| : F(t) = 0\}
\end{equation}
\end{proposition}

\begin{proof}
The loss is $\ell(t) = -z_y(t) + \log F(t)$. The linear term $-z_y(t)$ is
entire; singularities come only from $\log F$, which has branch points
where $F = 0$. If $F$ is nonzero on $|t| \le R$, then $\log F$ is
holomorphic there. At a zero $t_0$ with $|t_0| = \rho^*$, $\log F$ has a
branch point, so the radius is exactly $\rho^*$.
\end{proof}

\paragraph{Notation.}
In this section, $t$ denotes the possibly complex step parameter
(elsewhere written $\tau$ for real steps). With this convention, we
exploit complex analysis to find $\rho^*$ in closed form---starting
with binary classification.

\subsection{Binary Classification: Exact Formula}

For binary classification, we can compute $\rho^*$ in closed form. Let
$\Delta_a = |a_1 - a_2|$ be the logit-derivative spread and
$\delta = z_1(0) - z_2(0)$ the logit gap.

\begin{theorem}[Binary Convergence Radius]
\label{thm:binary}
For $F(t) = w_1 + w_2 e^{\Delta_a t}$ with $w_1, w_2 > 0$,
spread $\Delta_a = |a_1 - a_2|$, and logit gap $\delta = \log(w_1/w_2)$:
\begin{equation}
\boxed{\rho^* = \frac{\sqrt{\delta^2 + \pi^2}}{\Delta_a}}
\end{equation}
The minimum $\rho^* = \pi/\Delta_a$ occurs when $\delta = 0$ (balanced).
\end{theorem}

\begin{proof}
$F(t) = 0$ requires $e^{\Delta_a t} = -w_1/w_2$. Taking the complex
logarithm:
\[
\Delta_a t = \log(w_1/w_2) + i\pi(2k{+}1) = \delta + i\pi(2k{+}1)
\]
The nearest zeros (at $k = 0, -1$) lie at distance
$|t| = \sqrt{\delta^2 + \pi^2}/\Delta_a$.
\end{proof}

The formula reveals two contributions to the radius:
\begin{itemize}
\item The imaginary part $\pi/\Delta_a$ is confidence-independent: it is
set by $e^{i\pi} = -1$, the condition for exponential sums to cancel.
Even a balanced prediction ($\delta = 0$) has finite radius.
\item The real part $\delta/\Delta_a$ grows with confidence: large
margins shift the zero away from the imaginary axis, making confident
predictions more robust to large steps.
\end{itemize}

\subsection{Per-Sample Lower Bound}

For multi-class classification, exact computation of $\rho^*$ requires
finding zeros of $F(t) = \sum_{k=1}^n w_k e^{a_k t}$, which lacks a
closed form for $n > 2$. We derive a lower bound for a single sample
along a single direction, then aggregate to batches and datasets.

\begin{definition}[Logit-Derivative Spread]
\label{def:spread}
Along direction $v$, define the logit JVP $a(x; v) = J_{z(x)} v$ and
the \emph{per-sample logit-derivative spread}
\begin{equation}
\Delta_a(x; v) = \max_k a_k - \min_k a_k.
\end{equation}
\end{definition}

\subsubsection{Proof: \texorpdfstring{$\rho^*(x; v) \ge \pi/\Delta_a(x; v)$}{rho* >= pi/Delta\_a}}

We show that no zero of $F(t) = \sum w_k e^{a_k t}$
lies within distance $\pi/\Delta_a$ of the origin. For $n > 2$ classes,
the zeros have no closed form, so algebraic methods do not apply.
Instead, we use a geometric argument: when $|t|$ is small, all terms of
$F(t)$ lie in the same open half-plane and cannot cancel.

\begin{lemma}[Half-Plane Obstruction]
\label{lem:halfplane}
If complex numbers $u_1, \ldots, u_n$ have arguments in an open arc of
length $< \pi$, then $\sum c_k u_k \ne 0$ for any $c_k > 0$.
\end{lemma}

\begin{proof}
Let $\phi$ bisect the arc. Then $\mathrm{Re}(e^{-i\phi} u_k) > 0$ for
all $k$, so $\mathrm{Re}(e^{-i\phi} \sum c_k u_k) > 0$.
\end{proof}

\noindent Figure~\ref{fig:general-halfplane} visualizes this geometry:
vectors confined to an open half-plane share a projection direction with
positive real part, so no positive combination can produce zero.

\begin{figure}[!tb]
\centering
\includegraphics[width=\linewidth]{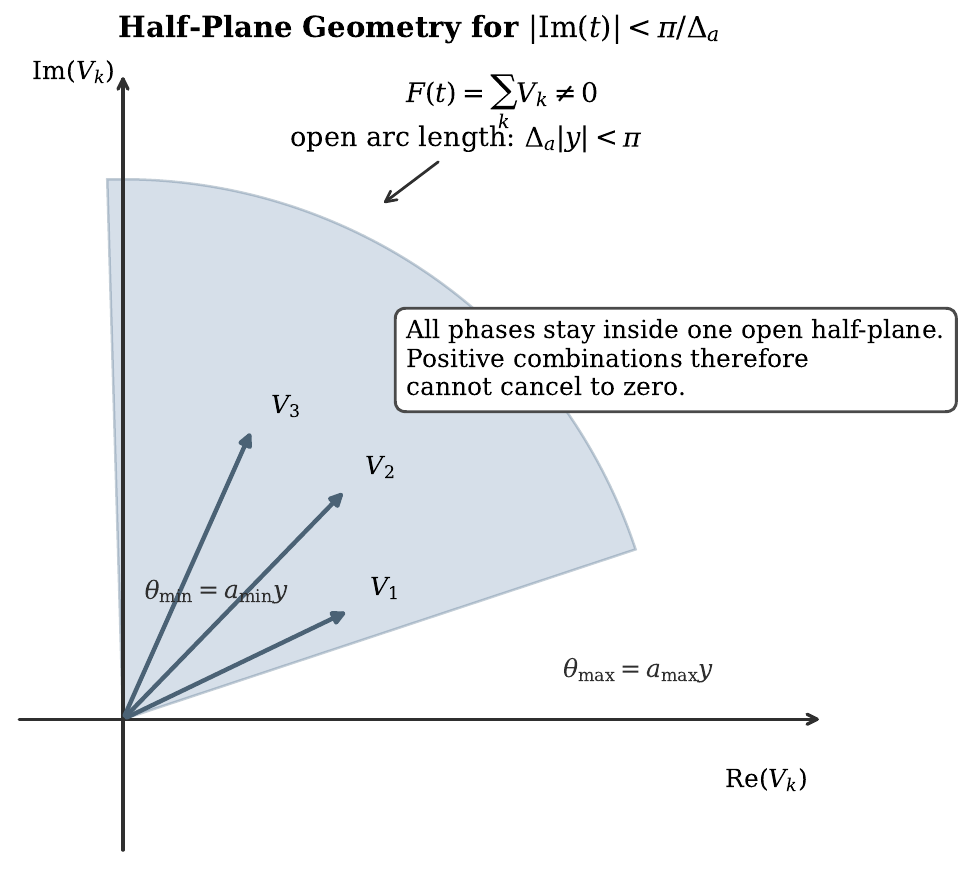}
\caption{Half-plane geometry behind Lemma~\ref{lem:halfplane}. If the
phases of $V_k = w_k e^{a_k t}$ lie in an open arc of length $< \pi$,
then all vectors remain in a common open half-plane, so no positive
combination can cancel to zero.}
\label{fig:general-halfplane}
\end{figure}

\begin{theorem}[General Lower Bound]
\label{thm:general}
For $F(t) = \sum w_k e^{a_k t}$ with $w_k > 0$:
\begin{equation}
\rho^* \ge \frac{\pi}{\Delta_a}
\end{equation}
where $\Delta_a = \max_k a_k - \min_k a_k$.
\end{theorem}

\begin{proof}
Let $t = u + iy$, where $u = \mathrm{Re}(t)$ and $y = \mathrm{Im}(t)$.
Each term $w_k e^{a_k t}$ has phase $a_k y$, because $w_k e^{a_k u} > 0$
contributes only magnitude. If $\Delta_a |y| < \pi$, all phases lie in
an arc of length $< \pi$. By Lemma~\ref{lem:halfplane}, $F(t) \ne 0$.
Therefore any zero satisfies $|\mathrm{Im}(t)| \ge \pi/\Delta_a$, giving
$|t| \ge \pi/\Delta_a$.
\end{proof}

Figure~\ref{fig:general-tplane} shows the same implication in the
complex $t$ plane: the strip $|\mathrm{Im}(t)| < \pi/\Delta_a$ is
zero-free, and the disk $|t| < \pi/\Delta_a$ lies entirely inside it.
How tight is this bound?

\begin{figure}[!tb]
\centering
\includegraphics[width=\linewidth]{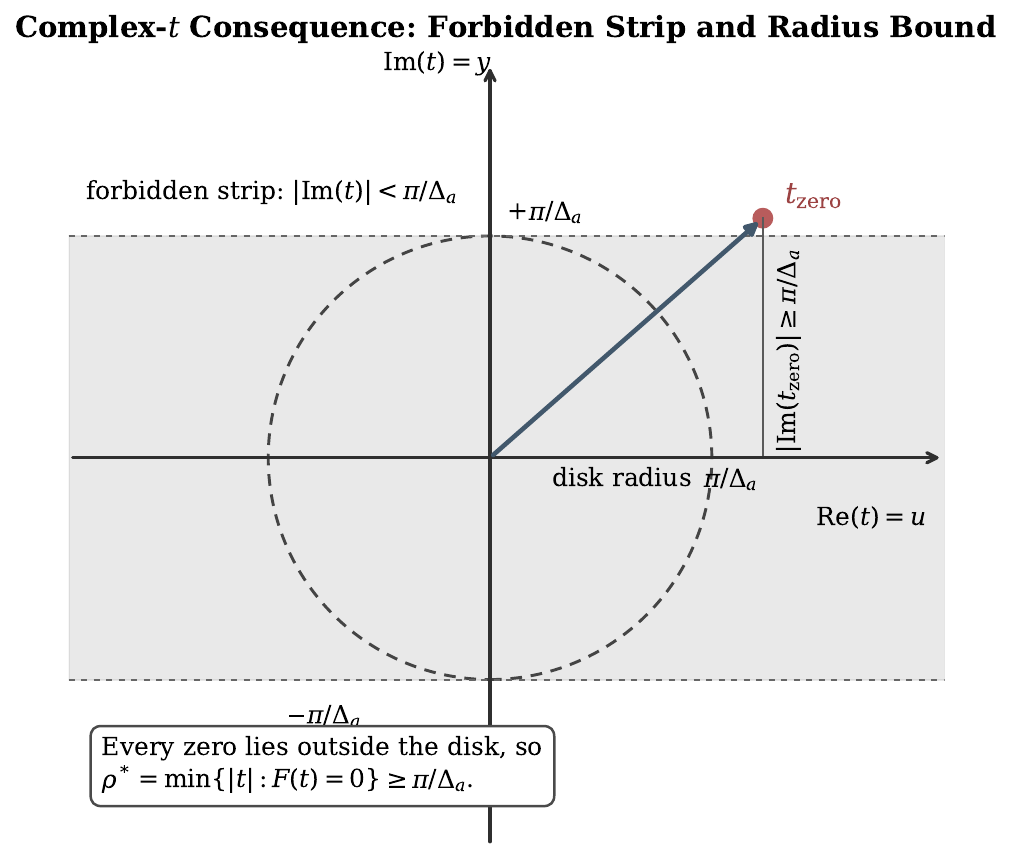}
\caption{Complex-$t$ consequence of Theorem~\ref{thm:general}. The
half-plane obstruction excludes zeros from the strip
$|\mathrm{Im}(t)| < \pi/\Delta_a$. Since the disk
$|t| < \pi/\Delta_a$ lies inside that strip, the minimal zero modulus
satisfies $\rho^* \ge \pi/\Delta_a$.}
\label{fig:general-tplane}
\end{figure}

\noindent For confident predictions ($|\delta|$ large) or multi-class
problems, the bound is conservative, but for binary classification at
$\delta = 0$ it is tight: $\rho^* = \pi/\Delta_a$.

\subsubsection{Per-sample ghosts}

Applying Theorem~\ref{thm:general} to a single sample~$x$ gives
$\rho(x; v) \ge \pi/\Delta_a(x; v)$. Each sample has a ghost---a
complex singularity determined by its logits and directional slopes.
For multiclass samples, exact ghost locations are roots of
$F_x(t)=\sum_k w_k(x)e^{a_k(x;v)t}$ and generally have no closed form,
but can be computed numerically for any given sample.
In the binary case, or under a top-2 reduction with margin
$\delta(x)=z_y-z_c$ and top-2 derivative gap
$\Delta_{y,c}(x;v)=a_y-a_c$, the nearest ghost is
\begin{equation}
\tau_{\text{ghost}}(x; v)=\frac{\delta(x)+i\pi}{\Delta_{y,c}(x;v)},\qquad
\rho(x;v)=\frac{\sqrt{\delta(x)^2+\pi^2}}{|\Delta_{y,c}(x;v)|}.
\end{equation}

This formula explains why $\pi/\Delta_a$ is conservative: confident
samples ($|\delta|$ large) have exact radii exceeding the bound by a
factor of $\sqrt{1+(\delta/\pi)^2}$, while low-margin samples lie near
the boundary.

\subsection{Batch and Dataset Bounds}

Each training step updates all samples simultaneously. To ensure safety,
$\tau$ must satisfy $\tau < \min_x \rho(x; v)$. From
Theorem~\ref{thm:general}, we have
$\rho(x; v) \ge \pi/\Delta_a(x; v)$. This yields the conservative
sufficient condition $\tau < \pi / \max_x \Delta_a(x; v)$. The
bottleneck sample---the one with largest $\Delta_a$---determines the
safe step size.

\begin{definition}[Computable Radius Bound]
\label{def:rho-a}
For a sample set $\mathcal{S}$, the $\Delta_a$-based radius bound is
\begin{equation}
\boxed{\rho_a(v) = \frac{\pi}{\max_{x \in \mathcal{S}} \Delta_a(x; v)}}
\end{equation}
\end{definition}

\paragraph{Notation conventions.}
When arguments are omitted, $v$ defaults to the gradient
direction $\nabla f / \|\nabla f\|$, and the sample set defaults to the
full training set $\mathcal{D}$. Thus
$\Delta_a(x) = \Delta_a(x; \nabla f/\|\nabla f\|)$ and
$\rho_a = \pi / \max_{x \in \mathcal{D}} \Delta_a(x)$.
Restricting to a mini-batch $\mathcal{B} \subseteq \mathcal{D}$ gives
a less conservative bound $\rho_{\mathcal{B}} \ge \rho_a$, since the
maximum over fewer samples is smaller. In general:
\begin{equation}
\rho_a \le \rho_{\mathcal{B}} \le \rho(x;v)\quad\text{for any }x\in\mathcal{B}
\end{equation}
The training-set bound $\rho_a$ is the most conservative; the per-sample
bound $\rho(x;v)$ is the least. In practice, the bottleneck samples that
drive $\rho_a$ down are those with small top-2 margins, whose nearest
ghosts lie closest to the real axis.

\subsection{Computing \texorpdfstring{$\Delta_a(x; v)$}{Delta\_a(x; v)}}

We now turn to practical computation. Computing $\Delta_a(x; v)$
requires one forward-mode automatic differentiation (AD) pass
(a Jacobian-vector product) per
sample:
\begin{enumerate}
\item Compute update direction $v = \nabla f / \|\nabla f\|$
\item For each sample $x$, compute $a(x; v) = J_{z(x)} v$
\item Return $\Delta_a(x; v) = \max_k a_k - \min_k a_k$
\end{enumerate}
\textbf{Cost:} each JVP is comparable to one forward pass
($\sim$1.5$\times$ overhead). Since only the bottleneck sample determines
$\rho_a$, total cost for a batch of $B$ samples is approximately $1.5B$
forward passes.

\paragraph{Summary.}
We defined the exact radius $\rho^*$ and derived its closed form for
binary classification. We then established the computable bound
$\rho_a(v) = \pi / \max_x \Delta_a(x; v)$ and proved that
$\rho^* \ge \rho_a$. Section~\ref{sec:predict} shows how to use this
bound to predict and prevent instability.

\section{Using the Radius}
\label{sec:predict}

Section~\ref{sec:radius} established that the Taylor expansion converges only
when the step distance does not exceed the convergence radius. For any
direction $v$:
\begin{equation}
\boxed{\tau < \rho(v)}
\end{equation}
where $\tau = \|p\|$ is the step distance and $\rho(v)$ is the convergence
radius, satisfying $\rho(v) \ge \rho_a(v) = \pi/\max_x \Delta_a(x; v)$. We now
define a normalized step size, interpret it as a learning-rate constraint, and
design a controller that enforces the bound.

\subsection{The Normalized Step Size}

Let $p$ be the parameter update from any optimizer, with
$v = p/\|p\|$ as the unit step direction. Both the step distance $\tau = \|p\|$
and the radius $\rho_a(v) = \pi/\Delta_a(v)$ are measured along $v$, with
\[
\Delta_a(v)
= \max_{x \in \mathcal{S}}
\left(\max_k a_k(x;v) - \min_k a_k(x;v)\right),
\quad a(x;v)=J_{z(x)}v.
\]
Here $\mathcal{S}$ is the sample set used by the controller (mini-batch or full
training set). The ratio of step distance to radius is dimensionless:

\begin{definition}[Normalized Step Size]
\label{def:ratio}
\begin{equation}
\boxed{r(v) = \frac{\tau}{\rho_a(v)} = \frac{\|p\| \cdot \Delta_a(v)}{\pi}}
\end{equation}
where $v$ is any unit optimizer direction, $\|p\|$ is the step norm, and
$\Delta_a(v)$ is the set-maximum spread defined above. For gradient descent,
$v = \nabla f / \|\nabla f\|$ and $\|p\| = \eta\|\nabla f\|$, recovering
$r = \eta\|\nabla f\|\cdot\Delta_a/\pi$.
\end{definition}

Like a power series $\sum a_n x^n$ with radius $R$, $r<1$ ensures
convergence under the conservative bound $\rho_a$. Because $r$ normalizes
by the local radius, it is architecture-independent: it measures the step in
natural units set by local loss geometry. Its value partitions optimizer
behavior into three regimes:
\begin{itemize}
\item $r < 1$: Taylor series converges; polynomial model is reliable
\item $r \approx 1$: boundary regime; error bound diverges
\item $r > 1$: Taylor series may diverge; finite-order approximations lose
      guarantees
\end{itemize}

\subsection{Learning Rate as Geometric Constraint}

The condition $r < 1$ can be rewritten as
\begin{equation}
\eta < \frac{\rho_a}{\|\nabla f\|} = \frac{\pi}{\Delta_a \|\nabla f\|}
\end{equation}
Rather than being tuned freely, the learning rate is upper-bounded at each
step by local analytic structure.

One might ask whether adaptive optimizers already enforce this bound.
Adam scales updates by inverse gradient moments, addressing per-parameter
scale differences. This is distinct from the radius constraint, which bounds
the \emph{global} step distance $\tau$. A flat Hessian (small gradient
moments) does not imply a large radius; the ghosts remain at fixed imaginary
distance $\pi$.

\subsection{A \texorpdfstring{$\rho$}{rho}-Adaptive Controller}

Standard optimizers like Adam ignore the radius bound: they scale updates by
gradient moments, so as $\rho_a$ shrinks during training
(Section~\ref{sec:control}), they keep using step sizes that worked earlier
until they cross the boundary and diverge.

Can we build a practical controller that enforces $\tau \le \rho_a$? The
requirements are: (1) compute $\rho_a$ efficiently, (2) modulate the step
without new hyperparameters, and (3) generalize across architectures.

\paragraph{Design.}
A simple norm clip directly enforces the conservative safety condition.

\begin{proposition}[Radius Clip Enforces $r \le 1$]
\label{prop:rhoscale}
Let $p$ be a tentative optimizer update, with $\tau_0=\|p\|$ and
$v=p/\|p\|$. Define
\[
s=\min\!\left(1,\frac{\rho_a(v)}{\tau_0}\right),\qquad
\tilde p = s\,p.
\]
Then $\|\tilde p\|\le \rho_a(v)$ and therefore
$r(v)=\|\tilde p\|/\rho_a(v)\le 1$.
\end{proposition}

\begin{proof}
\[
\|\tilde p\| = s\tau_0
= \min\!\left(\tau_0,\rho_a(v)\right)
\le \rho_a(v).
\]
Divide by $\rho_a(v)$.
\end{proof}

In practice, the controller rescales the tentative optimizer update:
\begin{equation}
s = \min\!\bigl(1,\; \rho_a / \|p\|\bigr),\qquad
\tilde p = s\,p
\end{equation}
where $\|p\|$ is the optimizer update norm (the Adam direction times the base
learning rate) and $\rho_a = \pi/\Delta_a$ is computed via one JVP along the
actual optimizer direction $v$. No additional hyperparameters are required.

\textbf{Framing.} This controller is a proof of concept, not a production
optimizer. It shows that $\rho$ alone contains enough information to
determine a conservative safe step size, without requiring manual retuning.

\section{Experimental Validation}
\label{sec:control}

The preceding sections derived a stability bound ($r < 1$) and designed an
adaptive controller. Three questions remain open: (1)~Does $r = 1$ actually
predict failure across architectures and directions? (2)~Does the theory
correctly predict how temperature and other parameters shift the boundary?
(3)~Does $r$ track instabilities in realistic, unperturbed training? We
address these progressively, moving from controlled single-step tests to
multi-step training and from artificial perturbations to natural instabilities.

\subsection{Learning Rate (LR) Spike Tests}

\begin{figure*}[!tb]
\centering
\includegraphics[width=0.85\textwidth]{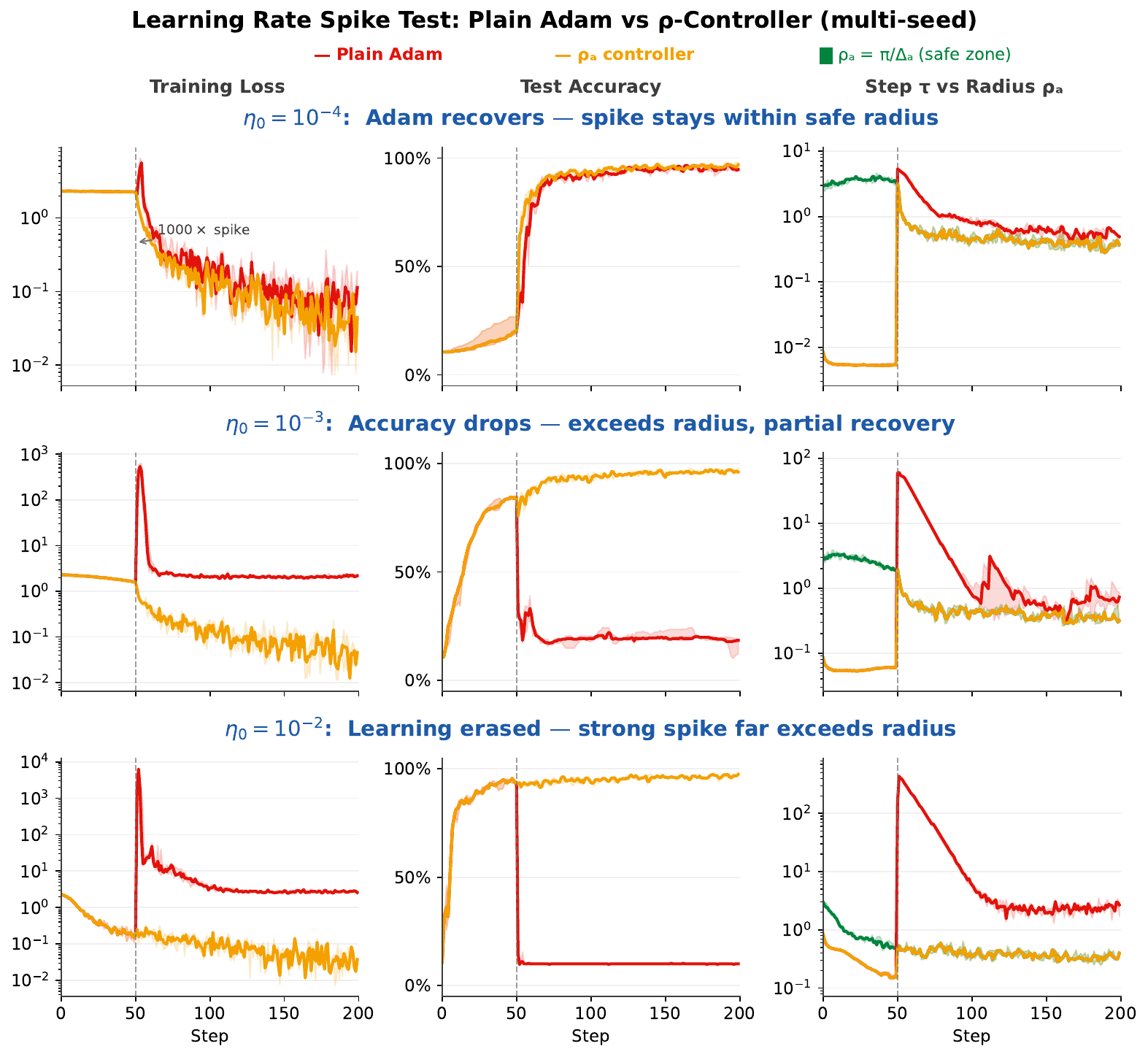}
\caption{Learning rate spike test (5 seeds, median with interquartile range (IQR) bands).
Each row uses a different base LR.
\textbf{Left}: training loss. \textbf{Center}: test accuracy.
\textbf{Right}: step size $\tau$ vs radius $\rho_a$ (green = safe zone).
Row~1: spike stays within $\rho_a$, Adam recovers. Row~2: $\tau$ exceeds
$\rho_a$ by ${\sim}10\times$, accuracy drops. Row~3: $\tau$
exceeds $\rho_a$ by ${\sim}100\times$, learning is erased. The
$\rho_a$-controller (gold) survives all spikes.}
\label{fig:lrgrid}
\end{figure*}

We begin with the most controlled setting: injecting a known LR spike and
asking whether $r = 1$ predicts the outcome. We sweep learning rates
over four orders of magnitude while injecting a $1000\times$ spike.

We use a two-layer multilayer perceptron (MLP; $64 \to 128 \to 10$) for digit classification with
Adam and custom step-size control. At step~50, the base learning rate is
multiplied by $1000\times$ and held for 150 more steps.

Figure~\ref{fig:lrgrid} shows three regimes:

\paragraph{$\eta_0 = 10^{-4}$ (row~1):} The spike raises the effective
learning rate to $0.1$, which remains within $\rho_a$. Both methods are
indistinguishable. The right panel confirms that $\tau$ stays in the green
safe region.

\paragraph{$\eta_0 = 10^{-3}$ (row~2):} Post-spike LR is $1.0$. Plain
Adam's accuracy drops to $\sim 50\%$, never recovering. The right panel
shows $\tau$ exceeding $\rho_a$ by $\sim 10\times$.

\paragraph{$\eta_0 = 10^{-2}$ (row~3):} Post-spike LR is $10$. Plain Adam
diverges---accuracy drops to chance. The right panel shows $\tau$
exceeding $\rho_a$ by $\sim 100\times$.

The $\rho_a$-controller (gold) survives all spikes by capping
$\tau \le \rho_a$. This is a controlled perturbation by design; the following
subsections test whether the boundary generalizes.

\subsection{Cross-Architecture Validation}

\begin{figure*}[!tb]
\centering
\includegraphics[width=\textwidth]{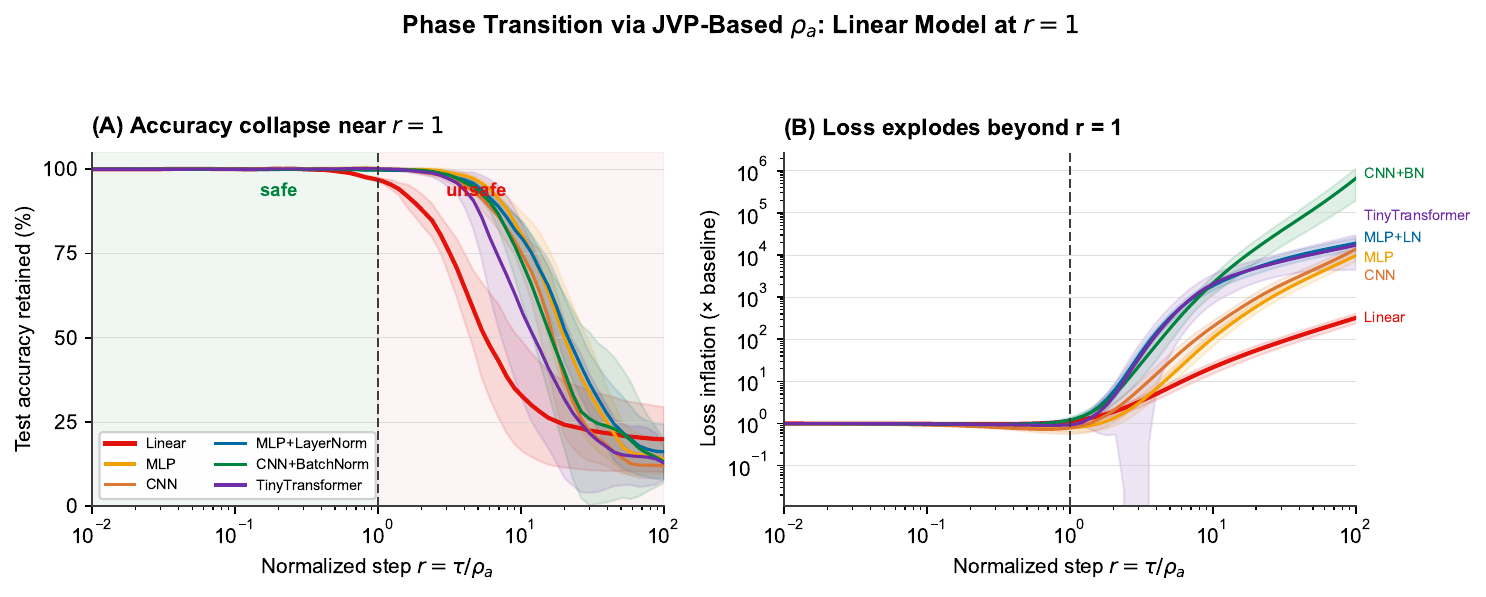}
\caption{Phase transition across architectures. Six architectures
trained to convergence, then one gradient step at varying $r$. (A)~Test
accuracy retained. Linear model transitions at $r \approx 1$. Other
architectures retain confidence-margin slack. (B)~Loss inflation
(post-step loss / pre-step loss).}
\label{fig:jvp}
\end{figure*}

Having established that $r = 1$ predicts spike-induced failure on a single
architecture, we ask whether the boundary holds more broadly.
Figure~\ref{fig:jvp} sweeps $r$ from $0.01$ to $100$ on six
architectures via single gradient steps on converged models.
The key result is that \textbf{no tested architecture fails for $r < 1$}.
Multi-step dynamics are examined in
Sections~\ref{sec:natural}--\ref{sec:transformer}.

Phase transitions cluster between $r = 1$ and $r = 10$. This slack reflects
confidence-margin slack: the exact radius
$\rho^* = \sqrt{\delta^2 + \pi^2}/\Delta_a$
(Theorem~\ref{thm:binary}) exceeds the conservative bound
$\rho_a = \pi/\Delta_a$ whenever the logit gap satisfies $|\delta| > 0$.
Confident predictions therefore extend the safe region beyond the
conservative bound. Before testing this quantitatively, we first verify that
$r$ is direction-independent.

\subsection{Random-Direction Validation}

\begin{figure*}[!tb]
\centering
\includegraphics[width=\textwidth]{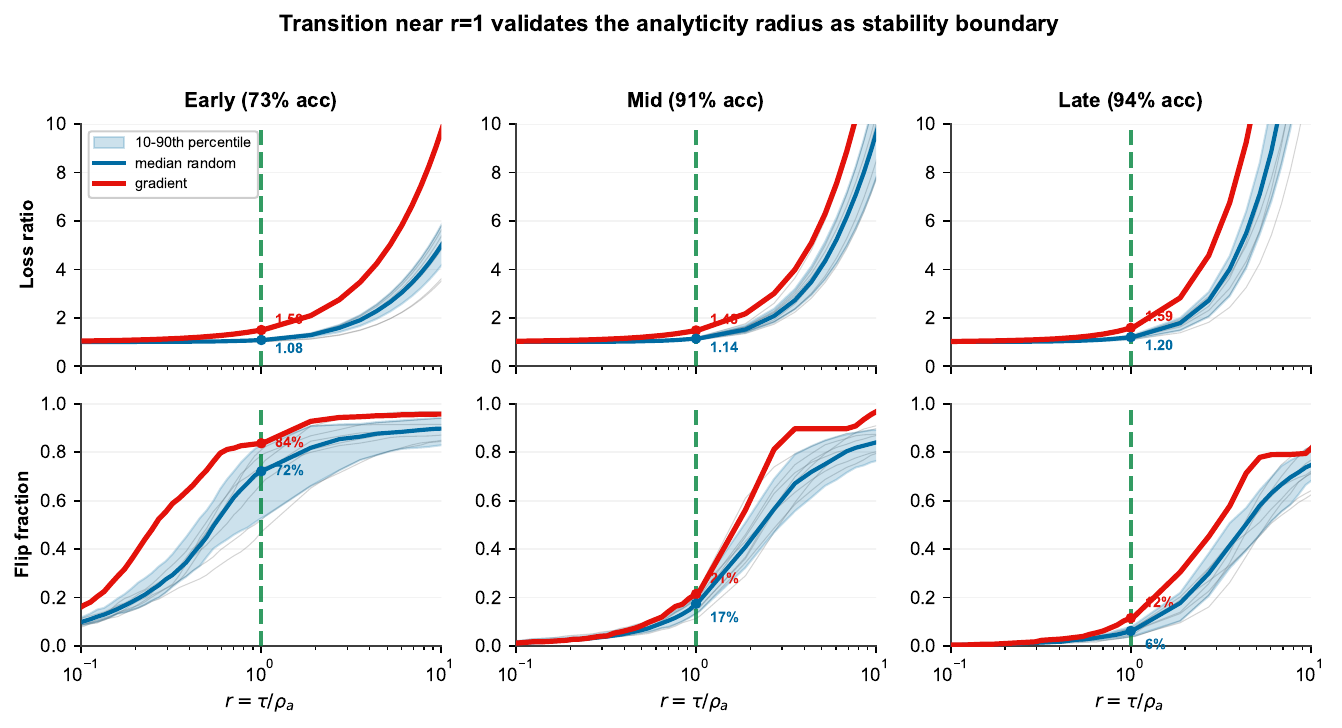}
\caption{Random-direction sweeps. Both loss ratio and flip fraction
transition near $r \approx 1$ across all tested phases, supporting $r$ as
direction-independent in these settings.}
\label{fig:random}
\end{figure*}

The previous tests perturbed along the gradient. Since $\rho_a = \pi/\Delta_a$
depends on direction through $\Delta_a$, we test whether normalizing by this
direction-specific radius yields a universal coordinate.
Figure~\ref{fig:random} sweeps $r$ along the gradient and along 20 random
directions at three training phases.

Both metrics---loss ratio (post-step loss divided by pre-step loss) and flip
fraction (proportion of predictions that change class)---transition near
$r \approx 1$ regardless of direction. The gradient direction is the extreme
case: at the same $r$, it shows earlier degradation than random directions.
Adversarial directions that minimize $\rho_a$ for fixed $\tau$ remain
unexplored. Nevertheless, the direction-dependent slack confirms that
$\rho_a$ is a conservative bound across all tested directions.

\subsection{Temperature-Scaling Fingerprint}

\begin{figure*}[!tb]
\centering
\includegraphics[width=\textwidth]{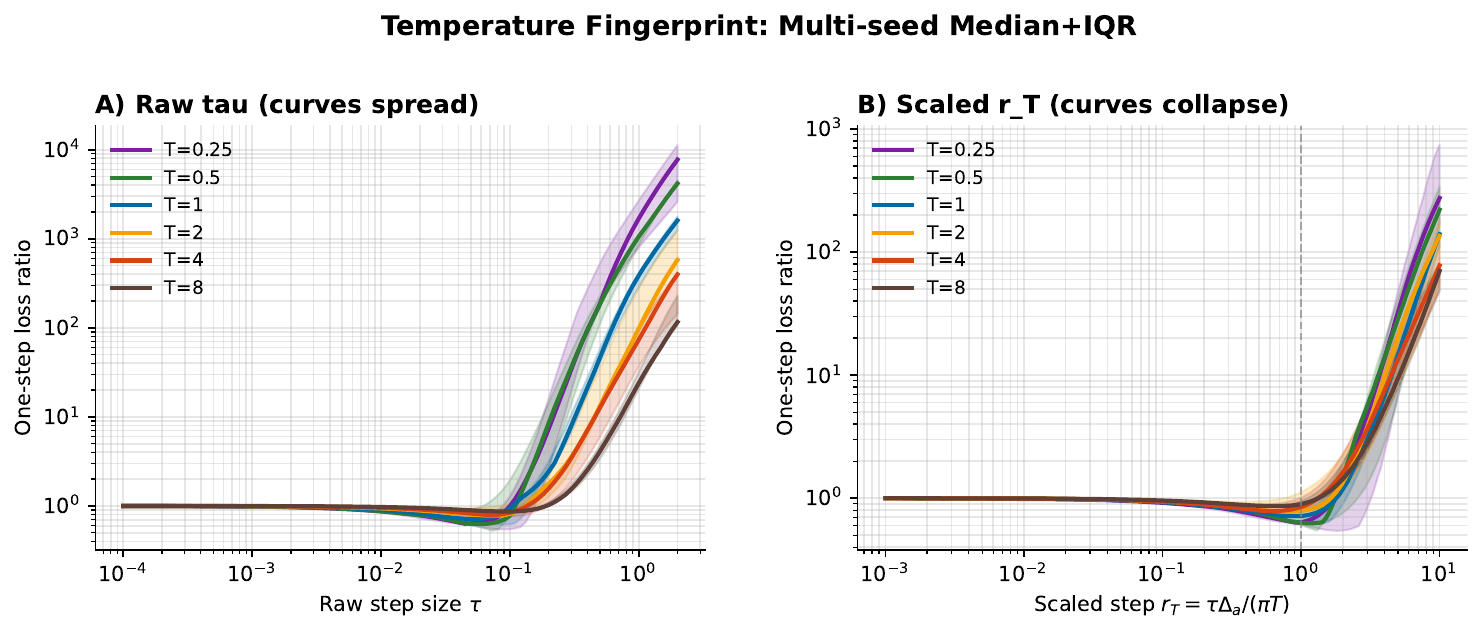}
\caption{Temperature-scaling fingerprint test. One-step loss inflation
(post-step loss / pre-step loss) is
measured across temperatures $T\in\{0.25,0.5,1,2,4,8,16,64\}$.
\textbf{A}: plotted against raw step size $\tau$, collapse thresholds vary
substantially across $T$ (collapse std $=0.992$). \textbf{B}: plotted against
the theory-normalized coordinate
$r_T=\tau\Delta_a/(\pi T)=\tau/\rho_a(T)$, curves align much more tightly
(collapse std $=0.164$).}
\label{fig:tempfingerprint}
\end{figure*}

The preceding tests confirmed the bound's shape; we now test a quantitative
prediction. Temperature rescales logits as $z/T$, so the predicted radius
rescales as $\rho_a(T)=\pi T/\Delta_a$.
Figure~\ref{fig:tempfingerprint} tests this prediction directly. In raw $\tau$
coordinates (panel~A), collapse onsets vary widely across temperatures
(collapse std $= 0.992$). After theory normalization by $T$ (panel~B),
those onsets collapse to a common transition (std $= 0.164$, a $6\times$
reduction).

This is a mechanistic check, not just a fit: changing $T$ shifts the boundary
in the direction and magnitude predicted by the analytic bound. The remaining
spread is expected from confidence-margin slack ($\rho^*>\rho_a$ for confident samples)
and nonlinear deviations from logit linearization.

\subsection{Controller Across Architectures}

\begin{figure*}[!tb]
\centering
\includegraphics[width=\textwidth]{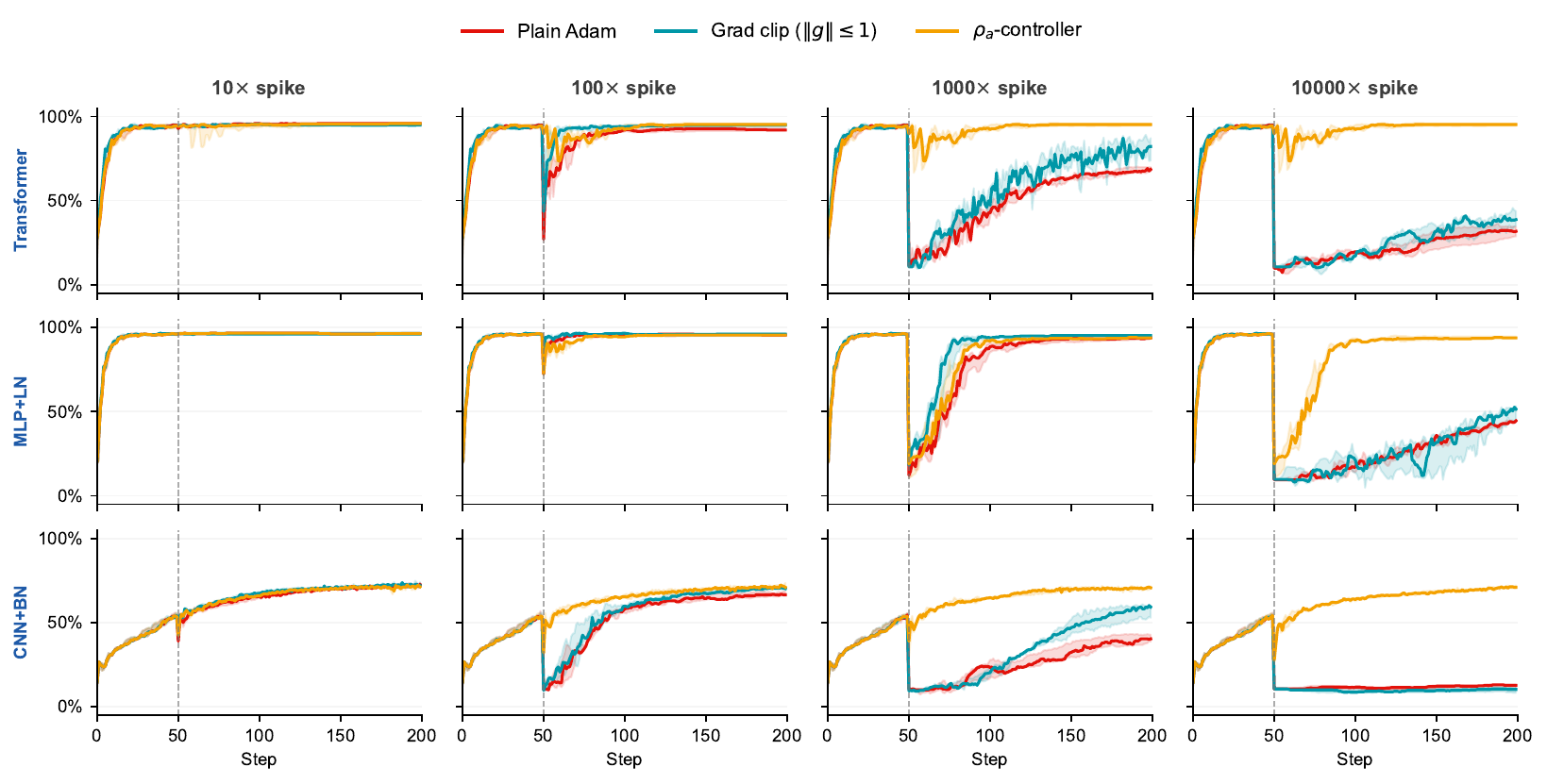}
\caption{Controller across architectures and spike severities
(5 seeds, median with IQR bands). Each column increases the spike
magnitude from $10\times$ to $10\,000\times$. Red: plain Adam. Teal:
Adam with gradient clipping ($\|g\|\le 1$). Gold: $\rho_a$-controller.
Gradient clipping helps partially at moderate spikes but still collapses
at high severity. The $\rho_a$-controller survives all conditions.}
\label{fig:archgrid}
\end{figure*}

Having validated the bound's shape, direction-independence, and temperature
scaling, we return to the controller. Extending the spike test to broader
architectures, we compare the $\rho_a$-controller against plain Adam and
gradient clipping ($\|g\| \le 1$).
Figure~\ref{fig:archgrid} tests three methods across three architectures
and four spike severities. As the spike magnitude increases, plain Adam (red)
fails progressively: the Transformer recovers from $10\times$ but not
$1000\times$; MLP with LayerNorm (MLP+LN) tolerates $100\times$ but collapses at $10\,000\times$;
a convolutional network with BatchNorm (CNN+BN) shows permanent collapse even at $10\times$ because BatchNorm's running
statistics become corrupted.

Gradient clipping (teal, threshold~1) provides partial protection---it
improves on plain Adam at moderate spikes---but it applies a uniform norm
threshold that does not adapt to the local radius $\rho_a$. At
$10\,000\times$, clipping still leaves Transformer and CNN+BN in high-loss
states. The $\rho_a$-controller (gold) survives all conditions by capping
$\tau \le \rho_a$.

\begin{table}[!tb]
\centering
\caption{Median final test accuracy at $10\,000\times$ spike.
A seed is divergent if final loss exceeds~10.}
\label{tab:archgridclip}
\begin{tabular}{lccc}
\toprule
Arch & Plain & Clip & $\rho_a$-ctrl \\
\midrule
Transformer & 31.8\% & 38.7\% & 95.3\% \\
MLP+LN & 44.6\% & 51.5\% & 93.9\% \\
CNN+BN & 13.4\% & 10.3\% & 71.6\% \\
\bottomrule
\end{tabular}
\end{table}

Table~\ref{tab:archgridclip} quantifies the $10\,000\times$ column.
Gradient clipping improves MLP+LN to $51.5\%$ but still falls far short of
the $93.9\%$ reached by the $\rho_a$-controller. For Transformer and CNN+BN,
clipping fails entirely. The geometric difference is the key: clipping
imposes a fixed norm threshold, whereas the $\rho_a$-controller scales the
update against the local radius. When the safe radius contracts unevenly,
a fixed threshold can still permit $r > 1$.

\subsection{Transformer Layer Analysis}
\label{sec:transformer}

\begin{figure*}[!tb]
\centering
\includegraphics[width=\textwidth]{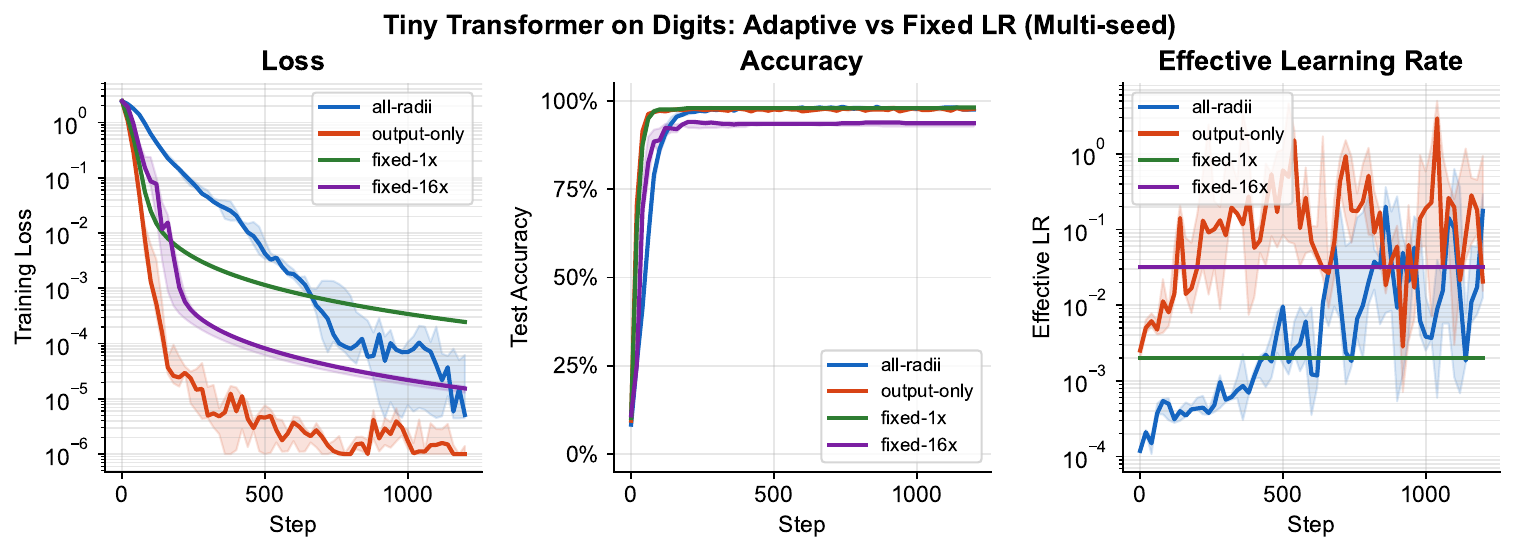}
\caption{Tiny transformer on Digits: adaptive vs fixed learning rates
(5 seeds, median with IQR bands).
Left: training loss. Center: test accuracy. Right: effective LR.
The adaptive controllers (all-radii, output-only) vary their step size
by orders of magnitude throughout training. Fixed LRs either converge
slowly (1$\times$) or diverge (16$\times$). The most conservative
strategy (all-radii, using $\min(\rho_{\text{out}},\rho_{\text{attn}},\rho_{\text{ffn}})$)
reaches the lowest loss and highest accuracy.}
\label{fig:transformer-radii}
\end{figure*}

The preceding experiments used a single output-logit radius. The checked-in
transformer artifact behind \Cref{fig:transformer-radii} compares four
strategies on a tiny transformer ($d=32$, 2 layers) trained on the Digits
dataset: two adaptive controllers and two fixed learning rates, with the
highest fixed rate set to $16\times$ the lowest. In that artifact, the
all-radii controller uses three local radii:
\begin{itemize}
\item $\rho_{\text{out}} = \pi / \Delta_{\text{out}}$: output logit
  spread (same as \Cref{sec:radius}).
\item $\rho_{\text{attn}} = \pi / \Delta_{\text{attn}}$: pre-softmax
  attention-logit spread (minimum across heads).
\item $\rho_{\text{ffn}} = Q_{0.01}(|h|/|\dot h|)$: a conservative
  FFN kink-distance proxy, defined as the $1\%$ quantile of
  $|h|/|\dot h|$ over FFN preactivations. This is the only included
  experiment that augments the softmax radii with a non-softmax term;
  see \Cref{sec:activations}.
\end{itemize}
The network radius is
$\rho_{\text{net}} = \min(\rho_{\text{out}}, \rho_{\text{attn}},
\rho_{\text{ffn}})$ for the all-radii controller, while the output-only
controller uses only $\rho_{\text{out}}$.

The right panel shows that the effective learning rate allowed by the
convergence-radius bound changes by orders of magnitude throughout training and
oscillates unpredictably. No fixed schedule can track this variation. A rate
that works at step~50 may violate the bound by step~150---heuristics work
until they fail.

The left panel shows the consequence. Fixed learning rates face an impossible
tradeoff: 1$\times$ converges slowly, while 16$\times$ diverges. The adaptive
controllers avoid this tradeoff by tracking the local radius. Remarkably, the
most conservative strategy
($\min(\rho_{\text{out}}, \rho_{\text{attn}}, \rho_{\text{ffn}})$)
reaches the lowest loss. Conservative steps compound; aggressive steps erase
progress.

\begin{table}[!tb]
\centering
\scriptsize
\setlength{\tabcolsep}{3pt}
\caption{Which radius binds in the all-radii transformer run
(\Cref{fig:transformer-radii}; 5 seeds, 61 logged evaluations per seed).
Counts report how often each term attained
$\rho_{\text{net}} = \min(\rho_{\text{out}}, \rho_{\text{attn}},
\rho_{\text{ffn}})$ within each training phase. Median radii are pooled over
all logged points in that phase.}
\label{tab:transformer-bottlenecks}
\begin{tabular}{lcccccc}
\toprule
Phase & FFN & Out & Attn & med.\ ffn & med.\ out & med.\ attn \\
\midrule
Early & 94/100 & 6/100 & 0/100 & 0.0164 & 0.0307 & 0.645 \\
Middle & 48/100 & 52/100 & 0/100 & 0.0111 & 0.0107 & 0.203 \\
Late & 8/105 & 97/105 & 0/105 & 0.0131 & 0.00793 & 0.168 \\
\bottomrule
\end{tabular}
\end{table}

\Cref{tab:transformer-bottlenecks} shows that the bottleneck pattern in this
artifact is not attention-first. Across all 5 seeds and 61 logged evaluations
per seed, $\rho_{\text{attn}}$ never attains the minimum. Early training is
FFN-limited, the middle third mixes FFN and output bottlenecks, and late
training is overwhelmingly output-limited. Every seed starts FFN-limited and
ends output-limited. So the observed transition in this artifact is
FFN early $\rightarrow$ output late.

\subsection{Natural Instability Detection}
\label{sec:natural}

\begin{figure*}[!tb]
\centering
\includegraphics[width=\textwidth]{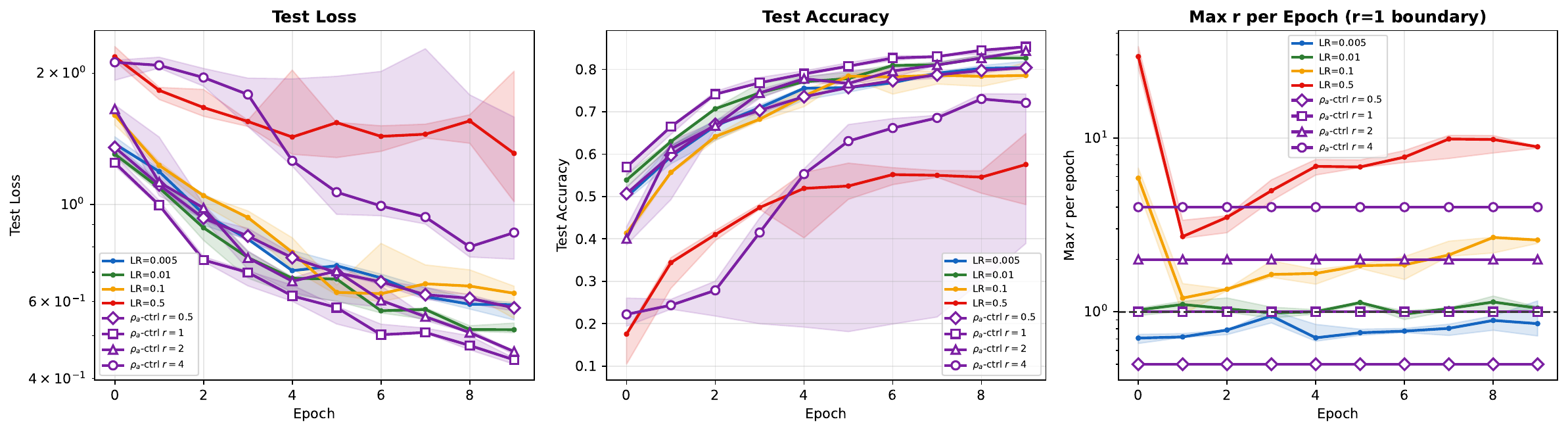}
\caption{Natural instability detection: ResNet-18 on CIFAR-10 with
SGD+momentum at four fixed learning rates and four $\rho_a$-controller
targets ($r{=}0.5, 1, 2, 4$; 5 seeds, median with IQR bands).
\textbf{Left}: test loss. \textbf{Center}: test accuracy.
\textbf{Right}: maximum $r = \tau/\rho_a$ per epoch with the $r{=}1$
boundary (dashed). Step size $\tau$ includes the momentum buffer.
The controller at $r{=}1$ reaches the highest accuracy ($85.3\%$);
$r{=}2$ remains stable; $r{=}4$ degrades with high variance. Even
$\eta{=}0.005$ occasionally exceeds $r{=}1$ due to momentum.}
\label{fig:cifarnatural}
\end{figure*}

The preceding experiments validate $r$ through artificial perturbations:
injected LR spikes, single-step sweeps, and controlled temperature shifts.
A natural question is whether $r$ also tracks instabilities that arise
organically during standard training---without any injected spike.

Figure~\ref{fig:cifarnatural} tests this on a larger-scale setting: ResNet-18
trained from scratch on CIFAR-10 with SGD and momentum~$0.9$. Four learning
rates span from conservative ($\eta = 0.005$) to aggressive ($\eta = 0.5$),
with $r = \tau/\rho_a$ logged every 50 steps using a finite-difference
estimate of $\rho_a$ on a current-batch subset (up to 256 samples).
The step size $\tau = \eta\|v\|$ uses the full momentum-corrected update
vector $v$, not the raw gradient, and we train for 10 epochs to focus on the
$r$--accuracy correlation (longer runs remain future work;
Section~\ref{sec:discuss}).

The results separate cleanly by $r$. At $\eta = 0.005$, the ratio stays
near~1 and median final accuracy is $80.4\%$. At $\eta = 0.01$,
occasional violations appear and accuracy reaches $82.6\%$. At $\eta = 0.1$,
violations are frequent and accuracy drops to $78.5\%$. At $\eta = 0.5$,
the run sits deep in the $r > 1$ regime and degrades to $57.5\%$.
Notably, even $\eta = 0.005$ exceeds $r{=}1$ once due to momentum
amplification---no fixed LR is unconditionally safe. Can the controller do
better?

Four $\rho_a$-controller targets confirm the bound's role. Crucially, the
controller uses no base learning rate: it sets $\eta = r\,\rho_a / \|v\|$
from the local geometry and the chosen target~$r$, replacing a base-LR
choice with a dimensionless aggressiveness target.
At $r{=}1$, the controller reaches the highest accuracy ($85.3\%$), exceeding
every fixed LR. At $r{=}0.5$ (conservative), it matches the best fixed rate
($80.4\%$). At $r{=}2$, accuracy remains strong ($84.4\%$), consistent with
the confidence-margin slack observed in the single-step experiments. At
$r{=}4$, accuracy drops to $72.1\%$ with large seed-to-seed variance
(IQR $38.9$--$74.2\%$), confirming that sufficiently exceeding the bound is
genuinely harmful.

Three aspects strengthen this result. First, no perturbation was injected---the
instabilities emerge from the natural interaction of learning rate, gradient
magnitude, and local loss geometry. Second, the optimizer is SGD rather than
Adam, so the same qualitative pattern appears outside the Adam-based spike
tests. Third, the shrinking-radius story still appears: as training sharpens
predictions, $\rho_a$ contracts and aggressive fixed learning rates move into
progressively higher-$r$ regimes later in training.
Table~\ref{tab:cifarnatural} quantifies the separation.

\begin{table}[!tb]
\centering
\caption{ResNet-18 CIFAR-10 natural instability summary (5 seeds, 10
epochs, momentum-corrected $\tau$). The controller at $r{=}1$ is optimal;
$r{=}2$ retains accuracy from confidence-margin slack; $r{=}4$ degrades.
Even $\eta{=}0.005$ exceeds $r{=}1$ due to momentum.}
\label{tab:cifarnatural}
\begin{tabular}{lcccc}
\toprule
Method & max $r$ & $r{>}1$ & Acc (med.) & IQR \\
\midrule
$\eta{=}0.005$ & 1.11 & 1  & 80.4\% & 80.0--81.9\% \\
$\eta{=}0.01$  & 1.23 & 13 & 82.6\% & 82.6--83.4\% \\
$\eta{=}0.1$   & 5.87 & 58 & 78.5\% & 78.1--81.7\% \\
$\eta{=}0.5$   & 29.34 & 70 & 57.5\% & 48.1--64.8\% \\
\midrule
$\rho_a$-ctrl ($r{=}0.5$) & 0.50 & 0  & 80.4\% & 80.2--80.5\% \\
$\rho_a$-ctrl ($r{=}1$) & 1.00 & 0$^\dagger$ & 85.3\% & 85.1--85.4\% \\
$\rho_a$-ctrl ($r{=}2$) & 2.00 & 80 & 84.4\% & 84.3--84.5\% \\
$\rho_a$-ctrl ($r{=}4$) & 4.00 & 80 & 72.1\% & 38.9--74.2\% \\
\bottomrule
\multicolumn{5}{l}{\footnotesize $^\dagger$ Strict exceedances are machine-precision noise only ($<10^{-15}$).} \\
\end{tabular}
\end{table}

No fixed rate keeps $r$ below~1 throughout training once momentum is
accounted for. As violations accumulate, accuracy degrades monotonically. The controller
results bracket the bound: $r{=}1$ is optimal, $r{=}2$ retains accuracy from
confidence-margin slack, and $r{=}4$ collapses.

\paragraph{Summary.}
The normalized step size $r = \tau/\rho_a$ reliably predicts instability in
our tested architectures and directions: $r < 1$ remained safe throughout
these experiments. The bound generalizes across architectures (six tested),
directions (gradient and 20 random), and optimizers (Adam and SGD), and
temperature normalization collapses onset thresholds with a $6\times$
reduction in spread. Beyond controlled perturbations, $r$ tracks instability
in natural ResNet-18/CIFAR-10 training without any injected spike, and
controller targets that bracket the bound ($r{=}0.5, 1, 2, 4$) show that
$r{=}1$ is optimal while degradation sets in beyond $r{=}2$. The bound is
conservative---most models survive to $r > 1$---but that conservatism is
explained by confidence-margin slack. Exploiting this bound, the
$\rho_a$-controller survives adversarial LR spikes up to $10{,}000\times$,
including regimes where gradient clipping collapses.

\section{Related Work}
\label{sec:related}

We connect several strands of optimization and deep-learning theory through a
single quantity: the \textbf{computable convergence-radius bound} $\rho_a$.
Each strand captures part of why training fails, but none identifies a common
mechanism or explains why instabilities erupt on apparently smooth landscapes.
The convergence-radius bound serves as both a theoretical lens and a cheap
diagnostic.

\paragraph{Stability and curvature.}
Cohen et al.~\cite{cohen2021gradient} observe the \emph{edge of
stability}: gradient descent keeps $\lambda_{\max}(H)\approx 2/\eta$.
Gilmer et al.~\cite{gilmer2022loss} link curvature spikes to
instability but use empirical sharpness rather than the underlying
analytic structure.
Classic theory demands global $L$-smoothness ($\eta<2/L$) or local
variants~\cite{zhang2020adaptive} that scale $L$ with the gradient
norm. These bounds control the quadratic approximation error. The convergence
radius controls the entire Taylor series rather than only its quadratic error.
That shifts the question from ``how wrong is the local model?''\ to ``does the
series converge at all at the proposed step?''

\paragraph{Step-size control.}
Natural-gradient methods~\cite{amari1998natural} and trust-region updates use
Fisher information or KL divergence to stay inside reliable regions; the
convergence-radius bound formalizes a similar idea in function space at the
cost of one JVP. Adam~\cite{kingma2015adam} and gradient
clipping~\cite{zhang2020adaptive,pascanu2013difficulty} treat the
\emph{symptoms} of instability without exposing the cause.
Classical rules such as Polyak's method~\cite{polyak1964some} and
Armijo line search~\cite{armijo1966minimization} adapt step length through
objective values or sufficient-decrease tests on~$\mathbb{R}$.
Our controller differs in mechanism: it estimates distance to a complex
singularity before the step, so the governing variable is not a descent test
on $\mathbb{R}$ but the analyticity boundary itself.

\paragraph{Scaling and parameterization.}
Maximal update parametrization ($\mu$P) and $\mu$Transfer~\cite{yang2022tensor}
aim to preserve useful hyperparameters across width and model scale. That is
complementary rather than competing with our goal. $\mu$P addresses how to
transfer hyperparameters across a family of models; $\rho_a$ measures local
step safety within one fixed model and training run.

\paragraph{Empirical training dynamics.}
Chowdhery et al.~\cite{chowdhery2022palm} report sudden loss spikes during
PaLM training; Touvron et al.~\cite{touvron2023llama} document similar shocks
in LLaMA training. Our interpretation is that these reflect steps that exceed
a shrinking~$\rho_a$.
Lewkowycz et al.'s ``catapult phase''~\cite{lewkowycz2020large}---a
brief divergence followed by convergence to flatter minima---is
consistent with $r$ moderately above~1 in our experiments. Very large $r$
values are much more likely to collapse runs.

\paragraph{Softmax and cross-entropy dynamics.}
Agarwala et al.~\cite{agarwala2020temperature} show that early learning under
softmax cross-entropy is governed by inverse temperature and initial logit
scale. Balduzzi et al.~\cite{balduzzi2017neural} analyze optimization in
rectifier networks through neural Taylor approximations.
Haas et al.~\cite{haas2025surprising} identify a ``controlled divergence''
regime: under cross-entropy, logits can grow unboundedly while loss, gradients,
and activations remain stable. Our framework offers a mechanism for this
phenomenon---partition zeros sit at fixed imaginary distance~$\pi$ regardless
of logit magnitude, so the loss surface appears smooth on~$\mathbb{R}$ even as
the convergence radius shrinks. To our knowledge, none of these works use the
distance to complex partition zeros as a step-size variable.

\paragraph{Theoretical foundations.}
The step-size limit stems from classical bounds on zeros of exponential sums
$\sum w_k e^{a_k t}$. The Lee--Yang theorem~\cite{leeyang1952} and extensions
by Ruelle~\cite{ruelle1971extension}, Tur\'{a}n~\cite{turan1953}, and
Moreno~\cite{moreno1973leeyang} locate such zeros. In numerical linear
algebra, Sao~\cite{sao2025trace} showed that convergence-radius constraints on
moment-generating functions cap trace-power estimators for log-determinants.
We transfer this insight to optimization: the partition function
$F(\tau)=\sum w_k e^{a_k\tau}$ underlying cross-entropy has complex zeros, and
those zeros set the Taylor convergence radius and hence the safe step size.

\paragraph{Singular learning theory.}
Watanabe's singular learning
theory~\cite{watanabe2010asymptotic,watanabe2013wbic} and related
work~\cite{yamada2011statistical} use zeta functions and analytic continuation
to characterize model complexity via poles of a learning coefficient (the real
log-canonical threshold). This shares our language of singularities and analytic
continuation but addresses a different question: generalization and model
selection, not optimizer step control. Their singularities are
algebraic---poles of a zeta function encoding parameter-space geometry. Ours are
transcendental---zeros of the partition function $F(\tau)$ in the complex
step-parameter plane.

\paragraph{Singularity-aware step control.}
The closest precedent lies outside machine learning, in numerical continuation
methods. Verschelde and Viswanathan~\cite{verschelde2022closest} use Fabry's
ratio theorem as a ``radar'' to detect nearby complex singularities and adapt
step size in homotopy path tracking. Telen et
al.~\cite{telen2020robust} use Pad\'{e} approximants for the same purpose:
locate the nearest singularity and set the trust region accordingly.
Timme~\cite{timme2021mixed} builds path-tracking controllers that exploit
distance to the closest singularity for precision and step-size decisions.
Our work transfers this principle---measure singularity distance, bound the
step---from polynomial system solving to neural network optimization, with
partition-function zeros of cross-entropy as the relevant singularities.

\paragraph{Continual learning.}
Elastic Weight Consolidation~\cite{kirkpatrick2017overcoming} and
natural-gradient methods~\cite{amari1998natural} guard prior knowledge with
Fisher-based importance. The convergence-radius bound offers a function-space
view: learned logits create singularities that shrink the region of reliable
updates. The scalar $\rho_a=\pi/\Delta_a$---one JVP---measures the remaining
headroom for safe learning.

\section{Discussion and Conclusion}
\label{sec:discuss}

\subsection{Summary}

We showed that cross-entropy loss carries a geometric constraint that standard
smoothness analysis does not capture: complex singularities of the partition
function cap the Taylor convergence radius along every update direction. Under
linearized logits---a tractability choice that makes the constraint both
computable and interpretable---the radius yields the bound
$\rho_a = \pi/\Delta_a$ (one Jacobian--vector product). For binary
cross-entropy the exact radius is
$\rho^* = \sqrt{\delta^2 + \pi^2}/\Delta_a$; the bound $\rho_a$ is the
worst case ($\delta = 0$) and applies to $n$ classes. The bound is
conservative: no tested architecture failed below it, but some survive
above it. As training sharpens predictions, $\Delta_a$ grows and the bound
tightens---explaining why late-training updates become fragile even as the
loss surface appears flatter on $\mathbb{R}$.

This bound has a simple interpretation. The loss landscape remains smooth on
$\mathbb{R}$, but beyond the convergence radius the Taylor series diverges, so
adding more terms worsens the approximation rather than improving it.
Step-size guarantees that rely on local polynomial extrapolation therefore lose
predictive power beyond $\rho_a$. Standard $L$-smoothness descent lemmas remain valid under their own
assumptions but do not encode this singularity geometry---$\rho_a$ captures
a separate, often stricter, constraint.

\subsection{Interpretation}
\label{sec:interpret}

\paragraph{Single-step hazard vs multi-step dynamics.}
Our bounds describe one-step reliability. They quantify local hazard, not full
training dynamics. A run can survive occasional high-hazard updates if the
iteration is self-correcting; however, repeated high-hazard updates can
accumulate through feedback and destabilize training. This distinction matters
for interpretation: $r$ is a risk scale, while long-run divergence depends on
how $r$ interacts with data conflict, architectural slack, and optimizer
dynamics.

\paragraph{Global controls vs local stability.}
Temperature and learning-rate schedules both act as indirect global controls,
while instability is governed by the local normalized step $r=\tau/\rho_a$.
When control targets the wrong variable (raw $\tau$ or a fixed global
$T$), updates can still violate $r>1$ on hard batches or directions. The
temperature-fingerprint experiment (Section~\ref{sec:control}) confirms this:
a fixed $T$ provides no per-batch safety guarantee, but normalization by
$\rho_a$ collapses onset thresholds across temperatures.

\subsection{Why One-Step Metrics Matter for Instability}

The one-step metrics we report (loss inflation, retained accuracy, and flip
fraction) directly measure whether an update preserved or damaged local model
reliability. Although these metrics capture single-step behavior, they matter
for iterative training because each update sets the next model state and the
next gradient. Large one-step disruptions can produce decision-boundary churn
that amplifies over subsequent updates.

This is why the transition near $r \approx 1$ is important: it provides a
practical operating boundary. We do not claim that $r>1$ always causes
divergence. Rather, crossing this boundary makes harmful one-step events more
likely, which increases instability risk when feedback dynamics are
unfavorable.

Together, these two levels connect theory and practice: $\rho_a$ provides a
computable local reliability scale, while one-step metrics show the
observable consequences of violating that scale.

\subsection{Why Training Tightens the Bound}

During training, the model separates logits to sharpen predictions,
causing $\Delta_a$ to grow. Since $\rho_a = \pi/\Delta_a$, the safe radius
shrinks. A learning rate that met $r < 1$ early in training may violate
this condition later.

Concretely, $\rho_a \approx 3.14$ when $\Delta_a = 1$, drops to $0.31$ when
$\Delta_a = 10$, and to $0.10$ when $\Delta_a = 30$.
With a fixed $\eta = 0.01$ and $\|\nabla f\| = 10$, the step
$\tau = 0.1$ gives $r = 0.03$ early but $r = 1.0$ later.

Rather than arbitrary ``cooling down,'' learning rate decay tracks the
contracting safe region. Warmup serves the opposite role. Early in training,
$\Delta_a$ is small (logits are similar), so $\rho_a$ is large. However, step
norms $\|p\|$ can also be large. Warmup caps $\eta$ to prevent
$\tau = \|p\|$ from exceeding this generous radius before updates stabilize.

\subsection{Hessian Intuition Versus Radius Geometry}

\begin{figure}[t]
\centering
\includegraphics[width=\columnwidth]{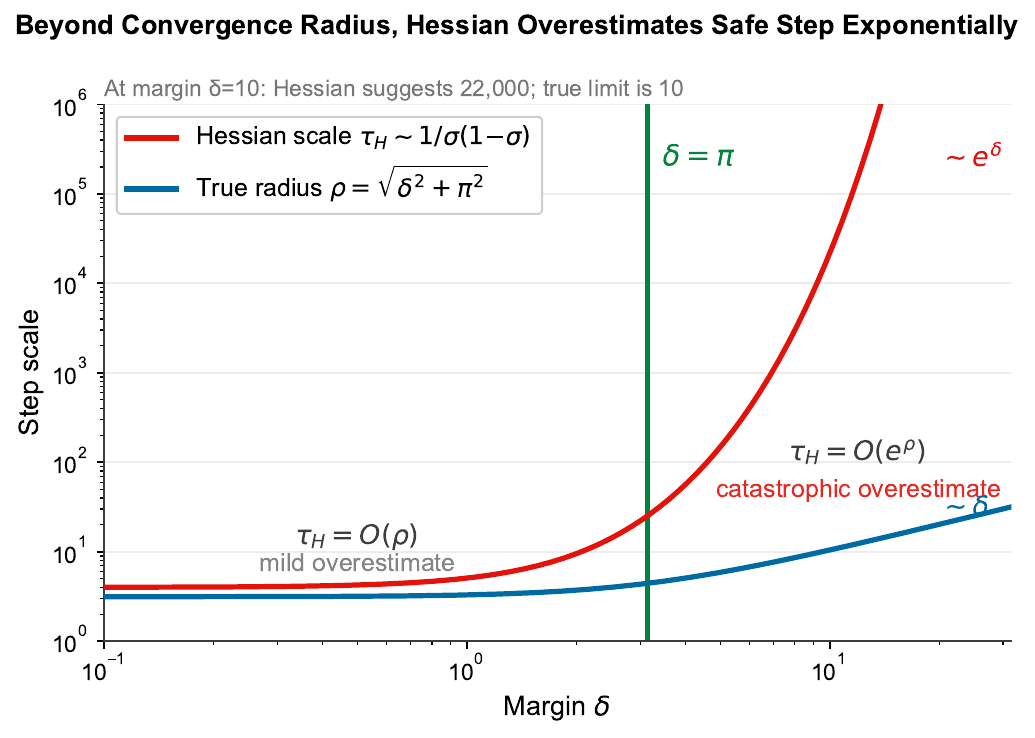}
\caption{Hessian curvature $\sigma(1-\sigma) \sim e^{-\delta}$ vanishes
exponentially while the radius bound $\pi/\sqrt{\delta^2+\pi^2}$ decays
only algebraically. At $\delta = 10$, the mismatch exceeds $4000\times$.}
\label{fig:crossover}
\end{figure}

Figure~\ref{fig:crossover} is best read as an interpretation of the binary
theory, not as a separate empirical result. For scalar logistic loss at margin
$\delta$, Hessian curvature decays like $\sigma(\delta)(1-\sigma(\delta))
\sim e^{-\delta}$, while the exact binary convergence radius from
Theorem~\ref{thm:binary} scales only like $\sqrt{\delta^2+\pi^2} \sim \delta$.
So a curvature-only view says confident examples become flatter and therefore
safer, whereas the radius view says Taylor reliability can still tighten
because the relevant complex singularities remain at imaginary distance $\pi$.

This resolves the apparent paradox from training practice: the loss surface can
look flatter on $\mathbb{R}$ even as updates become more fragile. Flatness on
the real line and distance to complex singularities are different notions of
safety, and cross-entropy forces them apart.

\subsection{Limitations}

\paragraph{Linearization.} Linearization is a tractability choice, not the
source of the phenomenon. The true loss $\ell(\tau)$ has a convergence radius
set by its own complex singularities regardless of how logits depend
on~$\tau$. Linearization yields the closed-form bound
$\rho_a = \pi/\Delta_a$ and reveals the controlling variable; for deep
networks where logits curve over the step, the true radius may be larger
or smaller than this estimate. Empirical validation
(Section~\ref{sec:control}) shows that the linearized bound remains
predictive in the tested settings.

\paragraph{Scale of validation.}
Our motivating examples come from frontier-scale training. As is standard for
mechanism-identification work, the direct evidence here uses controlled
small-scale systems (plus ResNet-18/CIFAR-10 and a tiny transformer) where
variables can be isolated. Testing the diagnostics at frontier scale is an
immediate next step.

\paragraph{Bound vs exact.} The lower bound $\rho_a = \pi/\Delta_a$ is
conservative by construction---it is the worst-case (balanced-logit)
specialization of the exact binary radius. Conservativeness is expected: a
lower bound on a convergence radius should not over-promise safety. Tightening
the bound via the sample-specific logit gap $\delta$ is a natural refinement.

\paragraph{Overhead.} On ResNet-18/CIFAR-10 (batch~128, RTX~6000 Ada), a
baseline SGD+momentum step (forward + backward + optimizer) takes
12.6\,ms. Finite-difference $\rho_a$ estimation raises this to 20.9\,ms
($+66\%$); exact JVP to 28.7\,ms ($+129\%$). Finite differences are
approximate but cheaper; JVP is exact. These numbers are for ResNet-18
only; overhead on large-scale models remains to be measured.

\paragraph{Multi-step dynamics.} The radius bounds single-step reliability,
not training trajectories. Intermittent $r > 1$ can cause loss oscillations
that self-correct; divergence in our experiments required $r > 1$
consistently over several iterations (Section~\ref{sec:interpret}).
Whether a run survives depends on data quality, architecture, optimizer
momentum, and schedule---variables beyond the one-step geometry.

\paragraph{Generalization.} The Taylor radius constrains local update
reliability, not what the model learns. It does not, by itself, prevent
overfitting or predict generalization. These depend on data quality,
regularization, and architecture choices that the radius does not capture.

\paragraph{Scope.} This paper focuses on softmax singularities: output
logits and attention pre-softmax scores. Activation functions and
normalization layers introduce additional singularities
(\Cref{sec:activations}), which can become the bottleneck in
unnormalized networks. We do not systematically test activation-function
singularities, though the checked-in tiny-transformer artifact includes a
conservative FFN kink proxy $\rho_{\text{ffn}} = Q_{0.01}(|h|/|\dot h|)$
from \Cref{sec:activations} as a single data point. That experiment is
FFN-limited early and output-limited late, while the attention-softmax radius
never binds. All experiments use small models (up to 128 dimensions);
validation on large-scale production transformers is still needed.

\paragraph{Computation.} We provide preliminary overhead numbers for
computing $\rho_a$ via JVP and finite differences; optimizing this
computation is not the focus. The bound should be computed over all
training samples, but this can be expensive; mini-batch estimates are
practical but approximate.

\paragraph{Higher-order extensions.} Linearization is not the only possible
proxy. Quadratic, cubic, or Pad\'{e} logit models can refine the estimate
(requiring Hessian-vector products), but the linearized model is already
valuable: it gives a closed-form mechanism, a cheap diagnostic, and the right
controlling variable~$\Delta_a$. The underlying principle---that complex zeros
determine the radius---does not change.

\paragraph{Controller.} The $\rho_a$-controller demonstrates that the bound
is actionable, not merely diagnostic. Integrating it with Adam's momentum,
modern schedulers, and production-scale workloads are natural engineering
extensions.

\subsection{Takeaways}

The central finding is that cross-entropy training operates under a
convergence-radius constraint set by complex singularities of the loss---a
constraint invisible to standard real-variable smoothness analysis.
Three practical consequences follow:

\begin{enumerate}
\item \textbf{Diagnosis.} If training diverges, compute
$r = \|p\|/\rho_a$ at that step, where $p$ is the optimizer update and
$\rho_a$ uses a JVP along $p/\|p\|$. If $r > 1$, the instability was
predictable from the analytic structure.

\item \textbf{Monitoring.} Track $\rho_a = \pi/\Delta_a$ during training.
This requires one JVP per step, costing about as much as a forward pass.
A shrinking $\rho_a$ signals that the current learning rate may become
unsafe; decay it preemptively.

\item \textbf{Adaptive control.} Cap $\tau \le \rho_a$ to guarantee
$r \le 1$. This reduces sensitivity to the learning rate: a wide range of
learning rates can be made safe by capping with the radius.
\end{enumerate}

The bound is conservative---models often survive $r > 1$---but no
architecture in our small-scale experiments failed below it. Large-scale
validation is needed before production use. The deeper point is structural:
the geometric constraint is intrinsic to cross-entropy and does not disappear
with better optimization. The ghosts are always there; the question is whether
the step reaches them.

\subsection{Future Work}

The most immediate next step is large-scale validation on models such as
Pythia and OLMo. We could also integrate our method with Adam, using
$\rho_a$ to modulate the step size after momentum and second-moment scaling.
Alternatively, regularization that penalizes large $\Delta_a$ could maintain
a larger bound and offer an alternative to post-hoc learning-rate reduction.

\printbibliography

\appendix
\section{From Linearized Theory to Real Networks}
\label{app:bridge}

The true loss $\ell(\tau)$ has its own convergence radius, set by its
nearest complex singularity. The linearized-logit model is useful because it
yields a closed-form softmax lower bound and makes the governing
quantity~$\Delta_a$ explicit. This appendix studies when that first-order
proxy remains reliable in deeper nonlinear networks.

\subsection{From the True Radius to a Tractable Proxy}

The key approximation replaces the true (nonlinear) logit trajectories
$z_i(\tau) = z_i(\theta + \tau v)$ with their linearizations
$z_i(0) + a_i \tau$. This is exact when only the output layer moves;
for deeper layers it introduces error depending on logit curvature.

The true univariate loss $\ell(\tau) = f(\theta + \tau v)$ has logits that
are nonlinear functions of $\tau$. Writing
$z_i(\tau) = z_i(0) + a_i \tau + r_i(\tau)$ where
$|r_i(\tau)| \le C_i|\tau|^2/2$ and $C_i$ captures the logit
Hessian along $v$, the linearized exponential sum
$F_{\mathrm{lin}}(\tau) = \sum w_i e^{a_i \tau}$ is perturbed to
$F_{\mathrm{true}}(\tau) = \sum e^{z_i(\tau)}$.

\subsubsection{When linearization is exact}
If the output layer is linear ($z = Wh + b$) and only output-layer
parameters move, then $z_i(\tau)$ is exactly linear in $\tau$. In this
case, $\rho^*$ is the true radius. For deeper layers, the composition of
nonlinear activations makes $C_i > 0$.

\subsubsection{Perturbation of zeros}
The linearized $F_{\mathrm{lin}}$ has zeros at imaginary distance
$\pi/\Delta_a$ from the origin (Theorem~\ref{thm:general}). We can
bound how far these zeros shift under the nonlinear perturbation.
Factor $F_{\mathrm{true}}$ as:
\begin{equation}
F_{\mathrm{true}}(\tau) = F_{\mathrm{lin}}(\tau) \cdot
\Bigl(1 + \frac{F_{\mathrm{true}}(\tau) - F_{\mathrm{lin}}(\tau)}{F_{\mathrm{lin}}(\tau)}\Bigr)
\end{equation}
For $|\tau| \le \pi/\Delta_a$, each remainder satisfies
$|r_i(\tau)| \le C\pi^2/(2\Delta_a^2)$ where
$C = \max_i C_i$. Define the conservative perturbation parameter
$\varepsilon = C\pi^2/(2\Delta_a^2 \cdot \min_i w_i)$. By Rouch\'{e}'s
theorem~\cite{rouche1862memoire,conway1978functions}, if $\varepsilon < 1$ then $F_{\mathrm{true}}$ and
$F_{\mathrm{lin}}$ have the same number of zeros inside the strip
$|\mathrm{Im}(\tau)| < \pi/\Delta_a \cdot (1 - \varepsilon)$. This gives:

\begin{proposition}[Linearization quality]
\label{prop:linqual}
Let $v$ be the unit step direction,
$\Delta_a = \max_i a_i - \min_i a_i$ the logit-derivative spread,
$\rho_a = \pi/\Delta_a$ the conservative linearized bound, and
$\rho_{\mathrm{true}}$ the true convergence radius.
If the logit curvature $C = \max_i \|d^2 z_i/d\theta^2\|$
along $v$ satisfies
$\varepsilon = C\pi^2/(2\Delta_a^2 \cdot \min_i w_i) \ll 1$, then
$\rho_{\mathrm{true}} \ge \rho_a(1 - \varepsilon)$.
\end{proposition}

The criterion $\varepsilon \ll 1$ says: the logit Hessian must be small
relative to $\Delta_a^2$, after accounting for the weakest exponential
weight in the local partition sum. Residual connections help by keeping
$C$ moderate: skip connections keep the logit response approximately
linear. Networks without residual connections, such as very deep CNNs,
can have large $C$, degrading the approximation.

\paragraph{Why moderate error suffices.}
The instability condition $r = \tau/\rho_a \gg 1$ is robust to constant
factors in $\rho_a$. If the linearized radius is off by a factor of $2$,
the qualitative prediction (safe vs.\ dangerous) is unchanged whenever
$r > 2$ or $r < 0.5$. The theory does not need $\rho_a$ to be exact; it
needs the right order of magnitude.

\subsection{Multi-Step and Stochastic Extensions}

\subsubsection{Single-step suffices for detection}
The theory provides a conservative \emph{sufficient} condition for
one-step reliability: $\tau < \rho_a$ keeps the step inside the
guaranteed local convergence region of the linearized softmax model.
Crossing this bound places the update in a hazard regime where the local
Taylor approximation need not converge, but it is not by itself a theorem
of full-training divergence. Capping $\tau \le \rho_a$ at every step
therefore enforces a local one-step safety condition, not a proof of
global convergence.

The theory does not guarantee multi-step convergence. After a safe
step, $\rho_a$ may decrease (as confidence grows). The controller handles
this by recomputing $\rho_a$ at each step.

\subsubsection{Stochastic gradients}
Mini-batch noise changes $v$, hence $\Delta_a$. Empirically,
$\Delta_a$ varies little across batches (coefficient of variation
${\sim}1\%$), so stochastic fluctuation is negligible.

\subsection{Failure Modes}

\subsubsection{Non-analytic activations}
ReLU networks are piecewise linear, not analytic. In the ReLU models
tested in this paper, the output-side softmax bound $\rho_a$ still
tracked instability usefully, suggesting that softmax ghosts can remain
an informative hazard scale even when activation kinks are present.
That is an empirical observation about the tested architectures, not a
general theorem that softmax singularities always dominate piecewise
linear obstructions. A rigorous extension to piecewise analytic or
kinked networks remains open.

\subsubsection{Scale}
All experiments use small models (up to 128 dimensions, 10 classes).
For large-vocabulary models (50{,}000+ classes), $\Delta_a$ will be
larger, making $\rho_a$ smaller. Whether large language models (LLMs) operate perpetually at
$r \gg 1$ through favorable cancellation, or whether the linearization
slack scales favorably with vocabulary size, is unknown.

\subsubsection{Comparison to simpler heuristics}
The framework improves on a generic ``reduce LR when logits are large''
heuristic in two ways: (i)~it provides the specific formula
$\rho_a = \pi/\Delta_a$ with the principled threshold $r = 1$, and
(ii)~it connects this threshold to Taylor convergence, explaining
\emph{why} large derivative spreads cause instability. Whether this
quantitative threshold provides practical advantages over simpler rules
at production scale remains to be tested.

\section{Activation Singularities}
\label{sec:activations}

Section~\ref{sec:radius} derived the loss radius by linearizing the
logits $z_k(\tau) \approx z_k(0) + a_k\tau$ and locating zeros of the
partition function. That linearization assumes every component between
parameters and logits---including activations---is analytic. Activation
functions can have complex singularities of their own, and these can cap
the loss radius before the softmax ghosts do. We apply the same lens to
feed-forward network (FFN) activations, which can introduce softmax-like
ghosts into the network.

\subsection{Per-neuron radius}

For a scalar activation $\phi$, let
$\Sigma_\phi \subset \mathbb{C}$ denote the singular or nonanalytic set
of its complex continuation. Consider neuron~$j$ with preactivation
$h_j \in \mathbb{R}$. Along the update direction, it evolves as
$h_j(t) = h_j + t\,\dot{h}_j$, and $\phi(h_j(t))$ first becomes
singular when $h_j + t\,\dot{h}_j \in \Sigma_\phi$. Here
$\dot{h}_j = \frac{d}{dt}h_j(\theta + t v)\big|_{t=0}
= \nabla_\theta h_j \cdot v$ is the directional derivative of the
preactivation. We compute all $\dot{h}_j$ values simultaneously using
one forward-mode AD Jacobian--vector product along $v$.
The activation-limited directional radius for that neuron is
\begin{equation}
\rho_j
= \min_{s \in \Sigma_\phi}
  \frac{|s - h_j|}{|\dot{h}_j|},
\label{eq:neuronradius}
\end{equation}
and the FFN radius is $\min_j \rho_j$.
Since the loss depends on the activations through composition, a
singularity in any neuron propagates to the loss. Thus the overall
convergence radius cannot exceed $\min_j \rho_j$.

Under the same logit linearization used throughout this paper,
the activation radius and the softmax radius
$\rho_a = \pi/\Delta_a$ give two separate upper bounds.
A conservative combined bound is
$\min(\min_j \rho_j,\;\rho_a)$.

Note the asymmetry: $\rho_a = \pi/\Delta_a$ carries a factor of~$\pi$
because partition-function zeros lie at imaginary distance exactly~$\pi$
(the condition $e^{i\pi}=-1$). The per-neuron radius~$\rho_j$ has no
universal~$\pi$; its scale is set by whichever singularity in~$\Sigma_\phi$
lies nearest to the current preactivation. For sigmoid that distance happens
to be~$\pi$, for tanh it is~$\pi/2$, and for ReLU it can be arbitrarily
small. The two radii therefore measure different geometric constraints:
$\rho_a$ bounds how far one can step before exponential sums in the partition
function cancel, while $\rho_j$ bounds how far before a hidden-layer
nonlinearity encounters its own complex singularity.

Comparing activation families therefore reduces to one question: where
is $\Sigma_\phi$?

\subsection{Activation families}

\paragraph{Piecewise activations}
(ReLU, Leaky ReLU, parametric ReLU (PReLU), Hard-Swish, Hard-Sigmoid,
ReGLU---a gated linear unit (GLU) variant with ReLU gate).
$\Sigma_\phi$ lies on the real axis at the breakpoints.
The local radius is the distance to the nearest breakpoint:
\[
\rho_j = \min_{b \in \mathcal{B}_\phi}\frac{|h_j - b|}{|\dot{h}_j|},
\]
For ReLU the main breakpoint is $0$, giving
$\rho_j = |h_j|/|\dot{h}_j|$.
In normalized networks many preactivations lie near break surfaces, so
$\min_j \rho_j$ is often small. Piecewise-linear functions are not
analytic at their kinks: they lack a Taylor series there entirely. This
is a stricter failure than complex singularities off the real axis.
This class is cheap and works well in practice, but from
the convergence radius lens it is structurally the worst.

\paragraph{Sigmoid family}
This includes sigmoid, softplus, sigmoid linear unit
(SiLU/Swish)~\cite{elfwing2018silu}, and SwiGLU
(Swish-gated linear unit)~\cite{shazeer2020glu}.
The logistic function $\sigma(z) = 1/(1+e^{-z})$ has poles where
$1 + e^{-z} = 0$, i.e.\ $z = i\pi(2k+1)$ for integer~$k$; SiLU gates
built from it share these poles. Softplus
$\operatorname{softplus}(z) = \log(1+e^z)$ has branch points where
$1 + e^z = 0$, i.e.\ at the same locations. The nearest singularity
lies at imaginary distance $\pi$---the same ghost lattice as the
softmax partition zeros. This family is much better than ReLU, but it
still imposes a hard finite cap from the $i\pi(2k{+}1)$ lattice. In the
language of this paper, these activations bring softmax-like ghosts into
the FFN.

\paragraph{Tanh family}
(tanh, anything built from tanh, including tanh-approximate Gaussian error
linear unit (GELU)).
$\tanh(z) = \sinh z / \cosh z$ has poles where $\cosh z = 0$,
i.e.\ $z = i(\pi/2 + \pi k)$.
The nearest pole is at imaginary distance $\pi/2$,
strictly closer than the sigmoid family.
This matters because many ``smooth ReLU'' implementations use
tanh-based approximations. In particular, exact GELU and
tanh-approximate GELU are analytically different objects: the
approximation loses the main structural advantage of exact GELU.

\paragraph{Entire activations}
(exact GELU~\cite{hendrycks2016gelu}, exact GeGLU
(GELU-gated linear unit), erf-based gates).
Exact GELU is $x\,\Phi(x)$, where $\Phi$ is the Gaussian cumulative
distribution function (CDF) built from
the entire function $\operatorname{erf}$. Thus
$\Sigma_\phi = \varnothing$: there is no finite activation singularity.
From the convergence-radius viewpoint, this class is structurally
strongest because the activation contributes no finite ghost barrier.

\subsection{Ranking}

Based on singularity distance alone, for vanilla FFNs:
\[
\begin{aligned}
\underbrace{\text{exact GELU / erf-based}}_{\Sigma=\varnothing}
&>\underbrace{\text{SiLU, softplus}}_{\pi} \\
&>\underbrace{\text{tanh}}_{\pi/2}
>\underbrace{\text{ReLU}}_{\text{real axis}}.
\end{aligned}
\]
For gated FFNs, the same ordering applies to the gate:
exact GeGLU or entire erf-gated GLU $>$ SwiGLU $>$ ReGLU.
The main caveat is that exact versus approximate GELU matters.

\paragraph{Caveat: entire $\neq$ unbounded step.}
A nonconstant entire function cannot be bounded on all of $\mathbb{C}$.
For example, $\operatorname{erf}(z)$ grows as $\exp(z^2)$ off the real
axis. So ``infinite activation radius'' does not mean arbitrarily large
steps are safe. It only means the activation itself contributes no
finite singularity barrier. The softmax radius $\rho_a$ and the growth
rate in the relevant complex strip remain as separate constraints.

\subsection{Exact vs.\ approximate GELU}

The common tanh approximation
\[
\operatorname{GELU}(x)
 \approx \tfrac{1}{2}x\bigl(
   1 + \tanh\!\bigl(\sqrt{2/\pi}\,(x + 0.044715\,x^3)\bigr)
 \bigr)
\]
reintroduces tanh poles.
This seemingly harmless shortcut introduces artificial complex
singularities that the exact implementation avoids. Many modern
frameworks provide exact erf-based GELU, but legacy codebases and some
hardware paths still use the tanh approximation.

\subsection{Designing radius-friendly components}
\label{sec:actdesign}

The ranking above suggests a design principle: prefer entire activations
(no finite singularities) over sigmoid-family or piecewise ones.
\Cref{app:actdesign} develops this idea, proposing RIA
(Rectified Integral Activation, $\operatorname{ReLU}$ convolved with a
Gaussian) for vanilla FFNs, GaussGLU (Gaussian cumulative distribution
function (CDF) gate) for gated FFNs,
and an analytic normalization layer via the Weierstrass transform.
These are suggestions from the theory, not claims of optimality.

\subsection{Limitations of this extension}

The per-neuron radius~\eqref{eq:neuronradius} is a necessary condition,
not a sufficient one. Cancellations across neurons could make the actual
loss radius larger. For ReLU, the framework requires reinterpretation
since piecewise-linear functions are not analytic. A full treatment,
including composition theorems for how activation and softmax radii
interact beyond linearization, is left to future work.

\section{Activation Design from the Radius Lens}
\label{app:actdesign}

Under the convergence-radius lens, an ideal activation would:
\begin{enumerate}
\item \textbf{\emph{Entire.}}
  Have no finite complex singularities, or at least singularities far
  from the real axis.
\item \textbf{\emph{ReLU-like on $\mathbb{R}$.}}
  Suppress the negative side while remaining approximately linear on
  the positive side.
\item \textbf{\emph{Monotone derivative.}}
  Avoid dead regions and the mild nonmonotonicity of exact Gaussian
  error linear unit (GELU).
\item \textbf{\emph{Controlled strip growth.}}
  Grow moderately for $z = x + iy$ when $|y|$ is in the range relevant
  to update steps.
\item \textbf{\emph{Cheap and stable to implement.}}
  Avoid tanh approximations that reintroduce poles.
\end{enumerate}
A nonconstant entire function cannot be bounded on all of $\mathbb{C}$.
So the target is not a globally bounded activation, but one that is
entire, ReLU-like on $\mathbb{R}$, and has mild growth in the relevant
complex strip.

\paragraph{Vanilla feed-forward network (FFN): RIA (Rectified Integral Activation).}
A first-principles design that satisfies these criteria is
RIA, the integral of the Gaussian cumulative distribution function (CDF)---equivalently,
$\operatorname{ReLU}$ convolved with a Gaussian:
\begin{equation}
\phi_\beta(x)
= \int_{-\infty}^{x} \Phi(\beta\, t)\,dt
= x\,\Phi(\beta x)
  + \tfrac{1}{\beta}\,\varphi(\beta x),
\label{eq:ria}
\end{equation}
where $\varphi$ is the Gaussian probability density function (pdf) and
$\Phi$ is its cumulative distribution function (CDF).
This is exactly $\operatorname{ReLU}$ convolved with a
Gaussian---the cleanest way to remove the kink while pushing
singularities to infinity.
Its properties:
$\phi_\beta'(x) = \Phi(\beta x) \in (0,1)$, so it is
monotone increasing;
$\phi_\beta''(x) = \beta\,\varphi(\beta x) > 0$, so it is
strictly convex;
$\phi_\beta(x) \to 0$ as $x \to -\infty$ and
$\phi_\beta(x) \sim x$ as $x \to +\infty$;
and since $\Phi$ and $\varphi$ are entire, $\phi_\beta$
is entire.
As $\beta \to \infty$ it recovers ReLU;
at finite $\beta$ it has the qualitative shape of
softplus but without the $i\pi(2k{+}1)$ ghost lattice.

\paragraph{Gated FFN: GaussGLU.}
For gated FFNs, using $g_\beta(x) = \Phi(\beta x)$ as the gate
instead of sigmoid defines GaussGLU:
\begin{equation}
\operatorname{GaussGLU}_\beta(x)
= (W_v\, x) \odot \Phi(\beta\, W_g\, x).
\label{eq:gaussglu}
\end{equation}
This is the radius-clean analogue of SwiGLU (Swish-gated linear unit):
SwiGLU inherits the logistic poles at $i\pi(2k{+}1)$, while GaussGLU
has no finite singularities from its gate.

\subsection{Analytic normalization}
\label{app:analyticnorm}

The same principle applies to normalization layers.
Layer normalization (LayerNorm) and root mean square normalization
(RMSNorm) use the scale factor
$f(v) = 1/\sqrt{v}$, which has a branch-point singularity
at $v = 0$.
Applying the Weierstrass transform (Gaussian convolution)
to the thresholded version
$f(v) = 1/\sqrt{\max(0,v)}$ yields
\begin{equation}
\tilde{f}_\sigma(v)
= \frac{1}{\sigma\sqrt{2\pi}}
  \int_0^\infty
  \frac{1}{\sqrt{t}}\,
  e^{-(v-t)^2/(2\sigma^2)}\,dt.
\label{eq:analyticnorm}
\end{equation}
Because the Gaussian is entire and the integrand is
absolutely integrable, this integral defines an
\emph{entire function} of~$v$. Consequently, the singularity at $v=0$
is removed.

\paragraph{Closed form.}
The integral~\eqref{eq:analyticnorm} equals
\[
\tilde{f}_\sigma(v)
= C\,e^{-v^2/(4\sigma^2)}\,
  D_{-1/2}\!\bigl(-v/\sigma\bigr),
\]
where $C$ is a constant and $D_{-1/2}$ is a
\emph{parabolic cylinder function}
(solution of Weber's equation).
Parabolic cylinder functions are entire: no poles,
no branch cuts.

\paragraph{Implementation.}
Deep-learning frameworks lack a native differentiable
parabolic cylinder function.
Three practical routes exist:
\begin{enumerate}
\item \textbf{\emph{Taylor series.}}
  Since $\tilde{f}_\sigma$ is entire, its Maclaurin series
  converges everywhere; precompute coefficients offline
  and evaluate a truncated polynomial at runtime.
\item \textbf{\emph{Randomized smoothing.}}
  Use the identity
  \[
  \tilde{f}_\sigma(v)
  = \mathbb{E}_{\xi \sim \mathcal{N}(0,\sigma^2)}
    \left[(\max(0,\,v{+}\xi)+\varepsilon)^{-1/2}\right];
  \]
  sample $\xi$ in the forward pass to recover the analytic
  function in expectation.
\item \textbf{\emph{Goldschmidt iterations.}}
  Approximate $1/\sqrt{v}$ by a fixed number of polynomial
  iterations; the $N$-step sequence is a polynomial and
  hence entire.
\end{enumerate}

\paragraph{Remark.}
The convergence radius framework \emph{suggests} the activation,
gating, and normalization designs above; it does not prove they are
optimal. Practical components depend on many factors beyond
analyticity, including gradient flow, trainability, and hardware
efficiency. Whether this theoretical advantage yields measurable
stability gains must be tested empirically.

\section{Kullback--Leibler (KL) Divergence and the Convergence Radius}
\label{app:kl}

The main text uses complex analysis to derive $\rho_a = \pi/\Delta_a$:
the Taylor series of the loss diverges when the step size $\tau$
exceeds the distance to the nearest complex singularity.

This appendix derives the same scale from a completely different
mechanism: real analysis of the KL divergence. The KL divergence
$\mathrm{KL}(P \| Q)$ measures the information lost when approximating
distribution $P$ by $Q$. For an optimizer step, it quantifies how much
the softmax output distribution shifts, directly measuring prediction
change.

We show that the quadratic approximation to KL is accurate when
$\tau$ is small on the same $O(1/\Delta_a)$ scale highlighted by the
ghost bound. The complex-analysis bound gives $\pi/\Delta_a \approx 3.14/\Delta_a$;
the KL crossover lies at or below $1/(4\Delta_a) = 0.25/\Delta_a$---a
roughly $12\times$ gap in the constant. Two unrelated derivations, one
using complex zeros and the other using real Taylor remainders, point to
the same $O(1/\Delta_a)$ controlling scale despite differing constants.
This confirms that $\Delta_a$ marks a genuine transition variable, not
an artifact of a single analysis method.

\subsection{Setup}

An optimizer step perturbs logits from $z$ to $z + \tau a$, where $a$
is the JVP direction and $\tau$ the step size. The softmax distribution
shifts from $p(0)$ to $p(\tau)$:
\begin{equation}
p_i(\tau) = \frac{e^{z_i + \tau a_i}}{\sum_j e^{z_j + \tau a_j}}
\end{equation}
We measure the resulting change in prediction with KL divergence.
Define the log-partition function
$K(\tau) = \log \sum_i e^{z_i + \tau a_i}$ and let
$\Delta_a = \max_i a_i - \min_i a_i$ be the logit-derivative spread.

\subsection{KL as a Bregman Divergence}

For exponential families, KL has an elegant form. Along the path
$\tau \mapsto p(\tau)$, it equals the Bregman divergence of the
log-partition:
\begin{equation}
\mathrm{KL}(p(\tau) \| p(0))
= K(\tau) - K(0) - \tau K'(0)
\end{equation}
This is the gap between $K(\tau)$ and its linear approximation at
$\tau = 0$. The derivatives of $K$ are the cumulants:
$K'(\tau) = E_\tau[a]$ (mean), $K''(\tau) = \mathrm{Var}_\tau(a)$
(variance), $K'''(\tau) = E_\tau[(a - \mu_\tau)^3]$ (skewness).

\subsection{Quadratic Approximation}

Taylor-expand $K(\tau)$ to second order with Lagrange remainder:
\begin{equation}
K(\tau) = K(0) + \tau K'(0) + \frac{\tau^2}{2}K''(0)
+ \frac{\tau^3}{6}K'''(\xi)
\end{equation}
Substituting into the Bregman formula, the constant and linear terms
cancel:
\begin{equation}
\mathrm{KL}(p(\tau) \| p(0))
= \frac{\tau^2}{2}\mathrm{Var}_{p(0)}(a)
+ \frac{\tau^3}{6}K'''(\xi)
\end{equation}
The leading term is quadratic---proportional to the variance of the
slopes under the base distribution. The cubic remainder determines
when this approximation fails.

\subsection{Sharp Bound on the Remainder}

Since $a$ lies in an interval of width $\Delta_a$, we can bound the
third cumulant $K'''(\xi) = E_\xi[(a - \mu_\xi)^3]$.

\begin{lemma}[Sharp third moment]
For any distribution on an interval of width $\Delta$,
$|E[(X - EX)^3]| \le \Delta^3/(6\sqrt{3})$.
\end{lemma}

\begin{proof}
The extremum is a two-point distribution at endpoints. Setting
$X \in \{0, \Delta\}$ with $P(X{=}\Delta) = p$, we get
$E[(X{-}EX)^3] = \Delta^3 p(1{-}p)(1{-}2p)$. Maximizing over $p$
gives $p^* = (3 \pm \sqrt{3})/6$, yielding $1/(6\sqrt{3})$.
\end{proof}

\begin{theorem}[Sharp KL remainder]
\label{thm:klbound}
For softmax distributions $p(\tau)$ with logits $z + \tau a$,
step size $\tau$, and logit-derivative spread
$\Delta_a = \max_i a_i - \min_i a_i$:
\begin{equation}
\boxed{
\left| \mathrm{KL}(p(\tau) \| p(0))
- \frac{\tau^2}{2}\mathrm{Var}_{p(0)}(a) \right|
\le \frac{|\tau|^3 \Delta_a^3}{18\sqrt{3}}
}
\end{equation}
\end{theorem}

The constant $1/(18\sqrt{3}) \approx 0.032$ is sharp.

\subsection{Connection to the Convergence Radius}

The quadratic approximation is accurate when the cubic remainder is
small compared to the quadratic term:
$|\tau|^3 \Delta_a^3 \ll \tau^2 \mathrm{Var}(a)$, i.e.\
$|\tau| \ll \mathrm{Var}(a)/\Delta_a^3$. Using
Popoviciu's inequality~\cite{popoviciu1935equations},
$\mathrm{Var}(a) \le \Delta_a^2/4$, so this crossover lies at or below
$1/(4\Delta_a)$. The KL analysis therefore points to an
$O(1/\Delta_a)$ scale, matching the order of $\rho/\pi = 1/\Delta_a$
but not the exact constant $\pi/\Delta_a$.

\paragraph{Two derivations, one scale.}
\begin{itemize}
\item \textbf{Complex analysis} (main text): Taylor series of
  $K(\tau)=\log\!\sum_i e^{z_i+\tau a_i}$ converges when
  $|\tau| < \pi/\Delta_a$
\item \textbf{Real analysis} (this appendix): KL quadratic
  approximation is accurate when $|\tau| \ll 1/\Delta_a$
\end{itemize}
Both involve $\Delta_a$ and break down at essentially the same scale.
The mechanisms are unrelated: one counts complex zeros, the other
bounds real Taylor remainders. They support the same order-of-magnitude
story, not the same sharp boundary.

\paragraph{Interpretation.}
Inside the softmax convergence radius, and more conservatively within
the KL quadratic regime, the loss and distribution shift admit simple
local approximations. The two analyses identify the same controlling
spread $\Delta_a$, but with different constants and guarantees. Together
they show that softmax departs from locally quadratic behavior on an
$O(1/\Delta_a)$ step-size scale.

\section{Why Hessian Bounds Miss Cross-Entropy}
\label{app:hessian}

The main text derives $\rho_a = \pi/\Delta_a$ as the stability limit for
cross-entropy optimization. A natural question is how this compares to
classical Hessian-based bounds. This appendix shows that Hessian bounds can be
wrong by orders of magnitude and explains why.

A common stability heuristic models the loss along a direction $v$ as quadratic:
\begin{equation}
L(\theta + \tau v) \approx L(\theta) + \tau g^\top v
+ \frac{\tau^2}{2} v^\top H v.
\end{equation}
For a true quadratic, gradient descent becomes unstable when the step
exceeds $2/\lambda_{\max}(H)$. This suggests a bound
$\tau \lesssim 2/(v^\top H v)$.

Experiments show that this Hessian-based scale can be wrong by orders of
magnitude, especially late in training. The issue is not just
non-quadraticity. As margins grow, the \emph{analytic structure} of
cross-entropy decouples from local Hessian curvature.

\subsection{Scalar Cross-Entropy Has Nearby Complex Singularities}

Consider the scalar logistic loss:
\begin{equation}
f(x) = \log(1 + e^{-x}),
\end{equation}
where $x$ is a logit margin. On the real line, $f$ is smooth and convex, with
\begin{equation}
f''(x) = \sigma(x)(1 - \sigma(x)) \le \frac{1}{4}.
\end{equation}
For a large margin $x = \delta \gg 1$, the Hessian decays exponentially:
\begin{equation}
f''(\delta) \sim e^{-\delta}.
\end{equation}
A quadratic model would therefore suggest a safe step size scaling as
$\sim e^{\delta}$.

However, $f$ has complex singularities where $1 + e^{-x} = 0$, located at
\begin{equation}
x = (2k + 1)i\pi.
\end{equation}
These are logarithmic branch points. The Taylor series of $f$ about a
real point $\delta$ has a finite radius of convergence:
\begin{equation}
R_x(\delta) = \sqrt{\delta^2 + \pi^2}.
\end{equation}
Although $f''(\delta)$ vanishes exponentially, $R_x(\delta)$ grows only
linearly with $\delta$. Consequently,
\begin{equation}
\frac{1/f''(\delta)}{R_x(\delta)}
\sim \frac{e^{\delta}}{\delta} \quad (\delta \to \infty).
\end{equation}

\paragraph{Key point.}
Local curvature governs the quadratic approximation, but convergence of
the full Taylor series is limited by the nearest complex singularity.
These quantities have fundamentally different asymptotic behavior.

\subsection{Training Direction Scales with \texorpdfstring{$R_x / |\Delta_{y,c}|$}{Rx / |Delta\_yc|}}

If the margin evolves as $x(\tau) = \delta + \Delta_{y,c}\tau$, the
per-sample loss is $\ell(\tau) = f(\delta + \Delta_{y,c}\tau)$. The
directional curvature and Taylor radius are:
\begin{equation}
\ell''(0) = f''(\delta)\Delta_{y,c}^2, \quad
\rho(\delta, \Delta_{y,c}) = \frac{\sqrt{\delta^2 + \pi^2}}{|\Delta_{y,c}|}.
\end{equation}
A Hessian-based step scales as $e^{\delta}/\Delta_{y,c}^2$, whereas the
true radius scales as $\delta/|\Delta_{y,c}|$. This mismatch grows
exponentially unless $|\Delta_{y,c}|$ compensates.

\subsection{Multiclass Case and Logit-JVP Gaps}

For multiclass cross-entropy, the same mechanism appears through the
top-2 reduction. Let $\delta = z_y - z_c$ be the margin and
$\Delta_{y,c} = a_y - a_c$ its directional slope, where $a(x; v)$ is
the logit Jacobian-vector product. The Taylor radius is:
\begin{equation}
\rho_i = \frac{\sqrt{\delta_i^2 + \pi^2}}{|\Delta_{y,c,i}|}.
\end{equation}
Confident samples ($\delta \gg 1$) contribute negligible Hessian
curvature yet still limit convergence at $\rho_i$.

\subsection{\texorpdfstring{$\Delta_a / \rho_a$}{Delta\_a / rho\_a} Predicts Loss Inflation Onset}

Define for a sample $x$:
\begin{equation}
\Delta_a(x; v) = \max_k a_k(x; v) - \min_k a_k(x; v).
\end{equation}
For a batch $\mathcal{B}$, the worst-case spread is:
\begin{equation}
\Delta_{a,\max}(v) = \max_{x \in \mathcal{B}} \Delta_a(x; v),
\end{equation}
and the corresponding radius:
\begin{equation}
\rho_a(v) = \frac{\pi}{\Delta_{a,\max}(v)}.
\end{equation}
The constant $\pi$ is exact: $1 + e^{\Delta_a \tau}$ has a zero at
$\tau = i\pi/\Delta_a$.

Empirically, loss inflation occurs near:
\begin{equation}
r = \frac{\tau}{\rho_a(v)} \approx 1,
\end{equation}
for random directions $v$. The Hessian-based scale $2/(v^\top H v)$ is
typically $10$--$100\times$ larger. Across 20 random directions at
three training stages (early, middle, and late), we find
$\rho_a < 2/\kappa$ in every case. The overestimation grows from
$\sim 10\times$ early in training to $\sim 460\times$ late in training,
consistent with $\tau_H/\rho_a \propto e^\delta$.

\subsection{Hessian-to-Ghost Crossover During Training}

The empirical findings above show Hessian overestimation growing from
$\sim 10\times$ to $\sim 460\times$ during training. We next ask when
the Hessian ceases to be binding.

We compare two candidate stability scales at each checkpoint:

\subsubsection{Quadratic/Hessian scale}
For a direction $v$, a quadratic model predicts a characteristic step size
\begin{equation}
\tau_H(v) = \frac{2}{\kappa(v)}, \qquad \kappa(v) = v^\top H v.
\end{equation}
For a worst-case bound, replace $\kappa(v)$ by $\lambda_{\max}(H)$.

\subsubsection{Ghost/\texorpdfstring{$\Delta_a$}{Delta\_a} scale}
Along the same direction $v$, define the logit Jacobian-vector product
$a(x; v) = J_{z(x)} v$ and its spread
$\Delta_a(x; v) = \max_k a_k - \min_k a_k$. The ghost scale is
\begin{equation}
\rho_a(v) = \frac{\pi}{\max_{x \in \mathcal{B}} \Delta_a(x; v)}.
\end{equation}

Empirically, one-step loss inflation begins when $\tau$ crosses
$\rho_a(v)$, i.e.\ when $r = \tau/\rho_a(v) \approx 1$, while
$\tau_H(v)$ often remains much larger late in training. To see when the
Hessian stops binding, consider the top-2 (binary) reduction for one
sample:
\begin{equation}
\delta(\tau) = \delta + \Delta_{y,c}\tau, \qquad
\ell(\tau) = \log(1 + e^{-\delta(\tau)}),
\end{equation}
where $\delta = z_y - z_c$ and $\Delta_{y,c} = a_y - a_c$.
The directional curvature is
\begin{equation}
\ell''(0) = \sigma(\delta)(1 - \sigma(\delta))\Delta_{y,c}^2,
\end{equation}
which decays as $e^{-\delta}$ for confident samples ($\delta \gg 1$).
Thus the quadratic scale $\tau_H \propto 1/\ell''(0)$ grows
exponentially with margin, while the ghost/$\Delta_a$ scale
$\rho_a \propto 1/|\Delta_{y,c}|$ lacks the saturating factor
$\sigma(1 - \sigma)$. Their ratio therefore grows rapidly with
confidence:
\begin{equation}
\frac{\tau_H}{\rho_a} \approx \frac{2 e^{\delta}}{\pi |\Delta_{y,c}|}.
\end{equation}
This predicts a crossover to ghost-dominated one-step reliability once
margins exceed
\begin{equation}
\delta^* \approx \ln\!\left(\frac{\pi |\Delta_{y,c}|}{2}\right),
\end{equation}
matching the transition point observed in our training-time tracking
(Figure~\ref{fig:crossover}).

\subsubsection{Edge-of-stability connection}
Across learning rates $\eta$, the crossover margin $\delta^*$ at
convergence follows $\delta^* \approx A + B \ln \eta$ with
$R^2 \approx 1$. This is consistent with an edge-of-stability-style
mechanism: increasing $\eta$ requires reducing an effective curvature
scale, which cross-entropy can achieve by inflating margins until
$\sigma(1 - \sigma)$ becomes sufficiently small. Since
$\sigma(1 - \sigma) \sim e^{-\delta^*}$, this yields
$\delta^* \propto \ln \eta$.

In this regime, Hessian-based quadratic scales matter less for one-step
reliability, while $\rho_a$ continues to predict the observed onset at
$r \approx 1$.

\subsection{Interpretation}

The Hessian measures curvature along the real line, which vanishes
exponentially for confident predictions. However, the true one-step
stability limit is set by complex singularities, captured by $\rho_a(v)$.

This explains why adaptive methods that scale by inverse curvature,
such as Adam, can become overly aggressive late in training: the
Hessian suggests large safe steps, while the convergence-radius bound
indicates a nearby singularity.

\paragraph{Connection to main results.}
This analysis complements the main text. Section~\ref{sec:radius}
derives $\rho_a$ from complex singularities, while this appendix
explains why Hessian-based alternatives fail. \Cref{app:kl}
provides a third derivation using KL divergence bounds. All three
approaches converge on the same scale: $1/\Delta_a$.

\end{document}